\documentclass[11pt]{article}
\PassOptionsToPackage{dvipsnames}{xcolor}
\usepackage[margin=1.25in]{geometry}

\usepackage[english]{babel}
\usepackage[sort]{natbib}
\renewcommand{\cite}{\citep}
\usepackage{times}

\usepackage{newpxtext,newpxmath}

\usepackage{amsthm}

\usepackage{
  mathtools,
  thmtools,
  url
}

\usepackage[
    colorlinks=true,
    linkcolor=blue,
    anchorcolor=blue,
    citecolor=blue
]{hyperref}

\newtheorem{theorem}{Theorem}[section]

\newtheorem{lemma}{Lemma}[section]

\newtheorem{fact}[lemma]{Fact}
\newtheorem{corollary}[lemma]{Corollary}
\newtheorem{definition}[lemma]{Definition}
\newtheorem{assumption}[lemma]{Assumption}

\newenvironment{remark}[1][Remark]
  {
  \begin{proof}[\textnormal{\textbf{#1}}]}
  {\end{proof}}

\ifdefined\usebigfont

\usepackage{times}
\usepackage[fontsize=13pt]{scrextend}
\makeatletter
\@ifpackageloaded{geometry}{\AtBeginDocument{\newgeometry{letterpaper,left=1.56in,right=1.56in,top=1.71in,bottom=1.77in}}}{\usepackage[letterpaper,left=1.56in,right=1.56in,top=1.71in,bottom=1.77in]{geometry}}
\AtBeginDocument{\newgeometry{letterpaper,left=1.56in,right=1.56in,top=1.71in,bottom=1.77in}}
\usepackage{hyperref} %
\else
\usepackage{times}
\fi

\title{Algorithmic Separation between Constant-Depth and Logarithmic-Depth Neural Networks}

\author{
  Yunwei Ren \\ Princeton University \\ \texttt{yunwei.ren@princeton.edu} \and 
  Zihao Wang \\ Stanford University \\ \texttt{zihaow@stanford.edu} \and 
  Jason D.~Lee \\ University of California, Berkeley \\ \texttt{jasondlee@berkeley.edu}
}

\usepackage{
  bbm, 
  bm,
  cancel
}

\usepackage{algorithm}
\usepackage{algpseudocodex}

\usepackage{tikz}
\usetikzlibrary{positioning,fit,shapes.geometric,calc,arrows.meta}
\usepackage{caption}
\usepackage[dvipsnames]{xcolor}

\usepackage[shortlabels]{enumitem}

\usepackage{thm-restate}

\newcommand\locallabel[1]{\label{\currentprefix:#1}}
\newcommand\localref[1]{\ref{\currentprefix:#1}}

\newcommand{\mbb}{\mathbb}
\newcommand{\mrm}{\mathrm}
\newcommand{\mbf}{\bm}
\newcommand{\mcal}{\mathcal}
\newcommand{\tnbf}[1]{\textnormal{\textbf{#1}}}

\newcommand{\eps}{\varepsilon}
\renewcommand{\a}{\mbf{a}}
\newcommand{\A}{\mbf{A}}

\newcommand{\C}{\mbb{C}}
\newcommand{\rd}{\mathrm{d}}
\newcommand{\e}{\mbf{e}}

\newcommand{\h}{\mbf{h}}

\newcommand{\M}{\mbf{M}}

\newcommand{\R}{\mathbb{R}}
\renewcommand{\S}{\mbb{S}}

\newcommand{\Id}{\mbf{I}}
\renewcommand{\u}{\mbf{u}}
\renewcommand{\v}{\mbf{v}}
\newcommand{\w}{{\mbf{w}}}
\newcommand{\W}{{\mbf{W}}}
\newcommand{\x}{{\mbf{x}}}

\newcommand{\y}{\mbf{y}}
\newcommand{\z}{\mbf{z}}

\newcommand{\cA}{\mcal{A}}
\newcommand{\cB}{\mcal{B}}

\newcommand{\cC}{\mcal{C}}

\newcommand{\cE}{\mcal{E}}
\newcommand{\cF}{\mcal{F}}
\newcommand{\cG}{\mcal{G}}
\newcommand{\cH}{\mcal{H}}

\newcommand{\cK}{\mcal{K}}

\newcommand{\bR}{\mbf{R}}

\newcommand{\cS}{\mcal{S}}

\newcommand{\cX}{\mcal{X}}
\newcommand{\bbN}{\mbb{N}}

\newcommand{\balpha}{\mbf{\alpha}}

\newcommand{\bPhi}{\mbf{\Phi}}
\newcommand{\bphi}{\mbf{\phi}}
\newcommand{\bpsi}{\mbf{\psi}}

\renewcommand{\Im}{\mrm{Im}}
\renewcommand{\Re}{\mrm{Re}}
\newcommand{\inprod}[2]{\left\langle #1, #2\right\rangle}
\newcommand{\norm}[1]{\left\|#1\right\|}
\newcommand{\abs}[1]{ {\left| #1 \right|} }
\newcommand{\braces}[1]{ \left\{ #1 \right\} }
\newcommand{\inv}{^{-1}}
\newcommand{\trans}{^\top}
\newcommand{\ps}[1]{^{(#1)}}
\newcommand{\indi}{\mathbbm{1}}

\newcommand{\E}{\mathop{\mathbb{E\/}}}
\renewcommand{\P}{\mathop{\mathbb{P\/}}}

\DeclareMathOperator{\poly}{poly}
\DeclareMathOperator{\diag}{diag}

\DeclareMathOperator{\arccosh}{arccosh}

\newcommand{\One}{\mbf{1}}

\newcommand{\Unif}{\mrm{Unif}}
\newcommand{\Proj}{\mbf{\Pi}}

\newcommand{\OP}{\mrm{OP}}
\newcommand{\Loss}{\mcal{L}}
\newcommand{\Tmp}{\texttt{Tmp}}

\newcommand{\Select}{\texttt{Sel}}

\newcommand{\HS}{\textnormal{HS}}
\newcommand{\Addr}{\textnormal{Addr}}

\begin{document}

\maketitle

\begin{abstract}
  Despite the empirical advantages of deep networks over shallow ones, theoretical depth separations largely concern approximation power, while algorithmic results are mostly limited to comparisons between two- and three-layer networks. 
  In this work, we prove the first algorithmic separation between constant-depth and logarithmic-depth networks.

  Specifically, we identify a class of Boolean functions with hierarchically structured Fourier spectra that logarithmic-depth networks can learn efficiently using layerwise coordinate descent by reconstructing the spectra hierarchically and adaptively. 
  We also exhibit a subclass for which every constant-depth, polynomial-width network with sufficiently regular activations and controlled spectral norms must incur constant $L^2$ approximation error under the uniform distribution over the hypercube. 
\end{abstract}

\section{Introduction}

One of the central questions in deep learning theory is how to formalize the advantage of deep models over shallow ones.
Beginning with the works \cite{eldan_power_2016,telgarsky_benefits_2016}, a long line of research has sought to separate the approximation power of deep and shallow networks \cite{daniely_depth_2017,safran_depth-width_2017,venturi_depth_2022,safran_depth_2025}.
These works exhibit functions that deep networks can approximate efficiently but shallow networks cannot.

However, due to the nonconvex nature of deep network training, being able to approximate a function does not imply that a network trained with gradient-based methods can efficiently find that solution. 
Therefore, it is unclear whether this approximation-based separation can properly explain the empirical advantage of deep networks. 
This motivates the pursuit of \emph{algorithmic} separation between deep and shallow networks. 
Namely, the goal is to find function classes that cannot be efficiently learned, or better, approximated by a shallow network, but can be efficiently learned by a deep network. 

In contrast to approximation-based separation, algorithmic-separation results are much rarer, and most concern models with two or three layers \cite{safran_optimization-based_2022,ren_depth_2023,nichani_provable_2023}.
In this work, we prove the first algorithmic separation between constant-depth and logarithmic-depth networks. 
Our target class consists of hierarchical staircase functions --- Boolean functions whose Fourier spectra have a hierarchical structure, in the sense that high-degree terms can be constructed by combining a constant number of low-degree terms.

We formally define hierarchical staircase functions in Section~\ref{sec: upper bound}. 
Here, as a running example and prototype of hierarchical staircase functions, we introduce the base deep quadratic function (cf.~Appendix~C of \cite{ren_provable_2026}). 
Suppose that $d = 2^{L_*}$ for some $L_* \in \bbN_+$. 
The base depth-$L_*$ deep quadratic function is defined as $f_*(\x) = L_*^{-1/2} \sum_{l=1}^{L_*} 2^{-(L_*-l)/2} f_l(\x)$, where $f_1(\x) = x_1 x_2 + x_3 x_4 + \cdots + x_{d-1} x_d$, $f_2(\x) = x_1 \cdots x_4 + \cdots x_{d-3} \cdots x_d$, $\dots$, $f_{L_*}(\x) = x_1 \cdots x_d$. 
General deep quadratic functions can be obtained by permuting the coordinates of this $f_*$ in each layer.
See Definition~\ref{def: deep quad function} for the formal definition. 

It is clear that for any deep quadratic function $f_*$, the total mass of terms with degree at least $\sqrt{d} = 2^{L_*/2}$ is at least $1/3$. 
Therefore, to learn $f_*$ to $o(1)$-accuracy, the learner network must capture those high-degree terms. 
The inability to do so is what we show in the lower bound against constant-depth networks. 

On the other hand, to learn $f_*$ using an $L$-layer network, it suffices to learn the features $(x_1 x_2, \dots, x_{d-1} x_d)$ in the first layer, combine them into the features $(x_1 \cdots x_4, \dots, x_{d-3}\cdots x_d)$ in the second layer using a quadratic mapping, and continue recursively. 
In other words, a deep network can learn the features \emph{hierarchically} and \emph{adaptively}. 

\begin{theorem}[Informal version of Theorem~\ref{thm: lower bound against constant-depth networks}, \ref{thm: upper bound} and \ref{thm: sep}]
  Suppose that $d = 2^{L_*}$. 
  \begin{itemize}
    \item \tnbf{(Lower bound against shallow networks)}
      No constant-depth fully connected network with sufficiently regular activations and $d^{c_0}$-spectrally bounded weights can approximate any depth-$L_*$ deep quadratic function, where $c_0 > 0$ is a sufficiently small constant.
    \item \tnbf{(Upper bound for deep networks)}
      An $(L_*+1)$-layer network, trained by layerwise one-pass coordinate descent, can efficiently learn the class of hierarchical staircase functions with bounded $L^2$-norm, which includes all depth-$L_*$ deep quadratic functions. 
  \end{itemize}
\end{theorem}

\subsection{Our contributions}

\paragraph{Polynomial approximation and lower bound.}

Let $f_*: \{\pm 1\}^d \to \R$ be such that the total Fourier mass of terms with degree at least $d^{1/2}$ is at least $1/2$.
Then, by the orthogonality of Boolean monomials under $\Unif(\{\pm 1\}^d)$, no polynomial of degree at most $o(d^{1/2})$ can approximate $f_*$ to $L^2$-error smaller than $1/2$. 
We prove the following polynomial approximation lemma for constant-depth networks, which may be of independent interest, and then derive the lower bound from it. 
\begin{lemma}[Informal version of Lemma~\ref{lemma: lb: approx f with poly}]
  Any constant-depth fully connected network with sufficiently regular activations and $C_W$-spectrally bounded weights can be approximated by a polynomial of degree $(C_W \log d)^{O(1)}$.
\end{lemma}

\paragraph{Hierarchical learning and upper bound.}
For our upper bound, we show that a depth-$L$ network can efficiently learn depth-$L$ hierarchical staircase functions, which are formally defined below. 
\begin{definition}[Hierarchical staircase function]
  Let $f: \{\pm 1\}^d \to \R$ be a Boolean function and $\hat{f}: 2^{[d]} \to \R$ be its Fourier transform. 
  Let $\cS_f := \braces{ S \subset [d] \,:\, \hat{f}(S) \ne 0 }$ denote the (Fourier) spectrum of $f$. 
  For $L, D \in \bbN$, we say that $f$ is a hierarchical staircase function of depth $L$ and local degree $D$ if $\cS_f \subset \cK_L$ with $(\cK_l)_{l \in \bbN} \subset 2^{[d]}$ defined inductively by 
  \begin{equation}
    \label{eq: def of cK}
    \cK_1 = \{ \varnothing, \{1\}, \dots, \{d\} \}, \quad 
    \cK_{l+1} 
    = \cK_l \cup \left(  \cK_l^{\oplus D} \cap \cS_f  \right),
  \end{equation}
  where $\cK^{\oplus D} := \{ K_1 \oplus \cdots \oplus K_D \;:\; K_1, \dots, K_D \in \cK \}$ with $\oplus$ denoting the symmetric difference.
  In addition, we call $|\cS_f|$ the width of $f$ and $\min_{S \in \cS_f} |\hat{f}(S)|$ the signal strength of $f$.

  For any $L, D, m \in \bbN_+$ and $\kappa > 0$, we use $\HS_{L, D, m, \kappa}$ to denote the collection of hierarchical staircase functions with depth at most $L$, local degree at most $D$, width at most $m$ and signal strength at least $\kappa$. 
\end{definition}

It is clear that $\cK_l \subset \cK_{l+1} \subset \cS_f \cup \{\varnothing, \{1\}, \dots, \{d\}\}$. 
In words, \eqref{eq: def of cK} means that at each step, we add to $\cK_l$ the sets in $\cS_f$ that can be constructed by combining $D$ existing terms in $\cK_l$. 
Thus, if every higher-order term in $\cS_f$ can be constructed from lower-order terms in $\cS_f$, possibly together with terms of degree at most $D$, then $\cK_l$ eventually contains $\cS_f$.
We initialize $\cK_1 = \{ \varnothing, \{1\}, \dots, \{d\} \}$ to include staircase functions such as $x_1 + x_1 x_2 + \cdots + x_1 \cdots x_L$ (\cite{abbe_staircase_2021,abbe_merged-staircase_2022,abbe_sgd_2023}). 

Also note that to learn hierarchical staircase functions, the model needs to learn the monomials not only hierarchically but also adaptively. 
Since there are potentially exponentially many monomials, the model cannot simply check if each of them belongs to $\cS_f$ even when $|\cS_f| = \poly(d)$.

\begin{theorem}[Informal version of Theorem~\ref{thm: upper bound}]
  For any $L, D, m \in \bbN_+, \kappa \in \R_+$ with $D = O(1)$, there is a depth-$L$ network that, when trained with empirical coordinate descent, can learn the class $\HS_{L, D, m, \kappa}$ of hierarchical staircase functions. 
  Moreover, the model size, sample complexity, and running time are all polynomials in $d, L, m, 1/\kappa$. 
\end{theorem}

\subsection{Related work}

\paragraph{Depth separation and lower bounds.}
Classical depth-separation results show that functions efficiently represented by deep networks may require much larger shallow networks \cite{eldan_power_2016,telgarsky_benefits_2016,daniely_depth_2017,safran_depth-width_2017,venturi_depth_2022,wang2024theoreticalanalysisinductivebiases,safran_depth_2025}. 
These results primarily concern approximation, while the comparatively few algorithmic separations focus on networks with two or three layers \cite{safran_optimization-based_2022,ren_depth_2023,nichani_provable_2023}. 
Our lower bound rules out the approximation of deep quadratic functions by constant-depth, polynomial-width networks.
Combined with our efficient learning guarantee, it yields an algorithmic separation between constant and logarithmic depth. 
Technically, most prior lower bounds rely on oscillation counting or special structural properties of two-layer networks. 
In contrast, our lower bound is based on approximating constant-depth networks by Chebyshev polynomials. 
The depth-$2$ lower bound in \cite{daniely_depth_2017} also uses the failure of low-degree polynomial approximation, but the argument relies heavily on harmonic analysis on the sphere and cannot be easily extended to more than two layers.

\paragraph{Hierarchical learning and upper bounds.}
Our upper bound is related to work showing that neural networks can progressively construct hierarchical features. 
This includes staircase and merged-staircase functions \cite{abbe_staircase_2021,abbe_merged-staircase_2022,abbe_sgd_2023,montanari2026phasetransitionsfeaturelearning}, two-level compositional targets learned by three-layer networks \cite{allen-zhu_what_2019,nichani_provable_2023,wang_learning_2024,fu_learning_2025,dandi_computational_2025,tabanelli_deep_2026}, and guarantees for genuinely deep hierarchies \cite{allen-zhu_backward_2023,daniely_deep_2026,ren_provable_2026,bruna_multiscale_2026}. 
Our hierarchical staircase condition allows each new Fourier feature to be composed from up to $D$ previously constructed features, so feature degree can grow exponentially with depth. 
This function class is much larger than the staircase functions considered in \cite{abbe_staircase_2021} and does not require the number of relevant coordinates to be constant as in \cite{abbe_merged-staircase_2022,abbe_sgd_2023}.
See Appendix~\ref{appendix: examples: staircase} for further discussion.
In addition, unlike the preceding hierarchical-learning results, our learning guarantee is paired with a lower bound showing that constant-depth networks (with potentially more than two layers) cannot learn the subclass of deep quadratic functions.

Our upper bound proof is built on the shallow-to-deep chaining principle (cf.~Informal Principle~2 of \cite{ren_provable_2026}). 
See Section~\ref{sec: upper bound} for more discussion. 
Our proof substantially extends the discussion in Appendix~C of \cite{ren_provable_2026}, which is informal and additionally assumes that features at the same level have disjoint supports.

\subsection{Notation and conventions}

For a vector $\x$, $\norm{\x}_p$ denotes its $\ell_p$-norm. 
When $p = 2$, we will often omit the subscript and write $\norm{\x}$. 
For a matrix $\W$, $\norm{\W}_{\OP}$ denotes its operator norm. 
For $a, b, \delta \in \R$, $a = b \pm \delta$ means $|a - b| \le \delta$. 
We define $[N] = \{1, \dots, N\}$ and $[N]_0 = \{0\} \cup [N]$ for any $N \in \bbN_{>0}$.
In addition, $\binom{[N]}{\le D}$ denotes the collection of subsets of $[N]$ of size at most $D$. 
For any $S \subset [d]$, let $\chi_S: \R^d \to \R$, $\x \mapsto \prod_{i \in S} x_i$ denote the multilinear monomial associated with $S$. 
The functions $(\chi_S)_{S \subset [d]}$ form an orthonormal basis of real-valued Boolean functions with respect to the uniform distribution over $\{\pm 1\}^d$.
We will only use this basic property in this paper. 
For more background on Boolean analysis, see \cite{odonnell_analysis_2014}.
Throughout this paper, the input distribution is always the uniform distribution over $\{\pm 1\}^d$ with $d \gg 1$.

\subsection{Outline}
The rest of this paper is organized as follows. 
We formally state our lower and upper bounds and sketch their proofs in Sections~\ref{sec: lower bound} and \ref{sec: upper bound}, respectively. 
Then in Section~\ref{sec: sep thm}, we instantiate the above bounds on deep quadratic functions and prove our depth-separation theorem.
We conclude in Section~\ref{sec: conclusion}. 
The appendix contains a table of contents, background on Chebyshev polynomials and approximation, complete formal proofs of our results, and more examples and discussion of hierarchical staircase functions.

\section{Lower bound against constant-depth networks}
\label{sec: lower bound}

In this section, we describe our lower bound against constant-depth networks and sketch its proof.
The complete proof can be found in Appendix~\ref{appendix: lower bound proofs}. 

\subsection{Setup and main result}

Let $\x \in \{\pm 1\}^d$ denote the input and set $\x\ps{1} := \x$ and $d_1 := d$. 
Then, we inductively define our learner network as 
\begin{equation}
  \label{eq: lb: network definition}
  \x\ps{l+1}(\x)
  = \bphi_l\left(\W\ps{l} \x\ps{l}(\x) \right), \quad 
  \forall l \in [L],  \quad
  f_L(\x) := \inprod{ \a }{ \x\ps{L+1}(\x) }, 
\end{equation}
where $\W\ps{l} = (\smash{ \w\ps{l}_k })_{k \in [d_{l+1}]}\trans \in \R^{d_{l+1} \times d_l}$ and $\a \in \R^{d_{L+1}}$ are the weights of the network, and $\bphi_l: \R^{d_{l+1}} \to \R^{d_{l+1}}$ are activation functions. 
We assume $\bphi_l$ are applied entrywise in the sense that $\bphi_l(\z) = ( \phi_{l, k}(z_k) )_{k \in [d_{l+1}]}$, where $\phi_{l, k}: \R \to \R$. 
Allowing coordinate-dependent activation functions lets us absorb the bias terms into the activations and use more compact notation.
Also, we use $\bphi_l'(\z) = ( \phi_{l, k}'(z_k) )_{k \in [d_{l+1}]}$ to denote the entrywise derivative of $\bphi_l$. 
We will assume the following about the network. 
\begin{assumption}
  \label{assumption: lower bound assumption}
  Consider the network defined in \eqref{eq: lb: network definition}. 
  Put $d_0 = d$ and $d_{\max} = \max_{l \in [L+1]_0} d_l$. 
  We assume that there are $C_W, C_\phi \ge 1$ and universal constants $c_\phi > 0$ and $m_\phi \ge 0$ such that the following conditions hold:
  \begin{enumerate}[(a)]
    \item $L$ is a universal constant; $d_{\max} \le d^{O(1)}$; and $C_W, C_\phi \le d^{c_0}$ with $c_0 < 2^{-2L-7}/L$.
    \item \label{itm: lb: op norm}
      $\norm{\W\ps{l}}_{\OP} \le C_W$ for all $l \in [L]$ and $\norm{\a} \le 1$.
    \item \label{itm: lb: sigma derivatives bounds}
      For every $l \in [L]$ and $k \in [d_{l+1}]$, $ |\phi_{l, k}(0)| \vee \big\|\phi_{l, k}'\big\|_{L^\infty} \vee \big\| \phi_{l, k}'' \big\|_{L^\infty} \le C_\phi$.
    \item \label{itm: lb: analytic, strip growth}
      For every $l \in [L]$ and $k \in [d_{l+1}]$, $\phi_{l, k}$ can be analytically continued to the strip $\{ z \in \C \,:\, |\Im\,z| < c_\phi \}$ and the continued $\phi_{l, k}$ satisfies the following strip-growth condition:
      \[
        \abs{ \phi_{l, k}(a + i b) } \le C_\phi ( 1 + |a| )^{m_\phi},  \quad 
        a, b \in \R, \, |b| \le c_\phi / 2.
      \]
  \end{enumerate}
\end{assumption}
\begin{remark}
  Condition~\ref{itm: lb: op norm} prevents the norm of intermediate representations from blowing up too quickly. 
  Under standard schemes such as Xavier initialization \cite{glorot_understanding_2010} and He initialization \cite{he_delving_2015}, this condition holds with $C_W = O(1)$ at initialization with high probability when adjacent layer widths are of the same order. 
  This condition is also compatible with the $\mu\mrm{P}$ and feature learning theory \cite{yang_tensor_2021,yang_spectral_2024}.
  
  Condition~\ref{itm: lb: analytic, strip growth} is used to show that $\phi$ can be approximated uniformly by Chebyshev polynomials on finite intervals (cf.~Lemma~\ref{lemma: chebyshev approx: strip}). 
  We verify this condition for common smooth activation functions, such as $\tanh$, sigmoid, softplus, and GELU (cf.~Lemma~\ref{lemma: chebyshev approx: common activations}).
\end{remark}

As noted in the introduction, our lower bound is based on polynomial approximation of the network.
The following is our main polynomial approximation lemma, which may be of independent interest. 
\begin{restatable}{lemma}{LBApproxFwithPoly}
  \label{lemma: lb: approx f with poly}
  Suppose that Assumption~\ref{assumption: lower bound assumption} holds.
  Then, for any $\eps_* \ge d^{-O(1)}$, there exists a polynomial $\tilde{f}_L: \{\pm 1\}^d \to \R$ such that $\norm{ f_L - \tilde{f}_L }_{L^2} \le \eps_*$ and 
  \[
    \deg \tilde{f}_L 
    \lesssim 
      \left( 
        (C_\phi C_W^2)^{2^{L+3}}
        (\log d)^{2L+2}
      \right)^{L(2^L+1)}
    \ll d^{1/2}.
  \]
\end{restatable}

The following theorem is a direct corollary of Lemma~\ref{lemma: lb: approx f with poly}.

\begin{theorem}[Lower bound against constant-depth networks]
  \label{thm: lower bound against constant-depth networks}
  Let $f_*: \{\pm 1\}^d \to \R$ be a Boolean function whose Fourier coefficients satisfy $\sum_{|S| \ge d^{1/2}} \hat{f}_*^2(S) \ge 1/3$. 
  Then, for any network $f$ satisfying Assumption~\ref{assumption: lower bound assumption}, we have $\norm{ f_L - f_* }_{L^2(\Unif(\{\pm 1\}^d))} \ge 1/4$ for all sufficiently large $d$.
\end{theorem}

\subsection{Proof sketch}

For ease of presentation, we assume in this subsection that $C_W$ and $C_\phi$ are also universal constants.
The general case is rigorously handled in Appendix~\ref{appendix: lower bound proofs}.
Here, we call a polynomial low degree if its degree is at most $\poly\log d$, and we use high degree to mean degree at least $d^{1/2}$.
As we have mentioned earlier, it suffices to show that $f_L$ can be approximated by a low-degree polynomial. 

To this end, we replace each non-polynomial activation $\bphi_l$ with its polynomial approximation. 
For each $l \in [L]$, put $\bar{\z}\ps{l} := \W\ps{l}\E\x\ps{l} \in \R^{d_{l+1}}$, and let $R_l \ge 1, \eps_\phi > 0$ be parameters to be chosen later. 
For every $k \in [d_{l+1}]$, let $\tilde\phi_{l, k}$ be the degree-$p_l$ Chebyshev approximation to $\phi_{l, k}$ on the interval $I\ps{l}_k := [ \bar{z}\ps{l}_k - R_l,  \bar{z}\ps{l}_k + R_l ]$, where the degree $p_l$ is chosen so that the $\norm{\cdot}_{L^\infty}$-difference between $\phi_{l, k}$ and $\tilde{\phi}_{l, k}$ over $I\ps{l}_k$ is bounded by $\eps_\phi$. 
See Appendix~\ref{sec: chebyshev} for background on Chebyshev polynomials and approximation. 
Then, define 
\begin{equation}
  \label{eq: lb: tilde x}
  \tilde\x\ps{l+1}(\x)
  = \tilde\bphi_l\left(\W\ps{l} \tilde\x\ps{l}(\x) \right), \quad 
  \forall l \in [L],  \quad
  \tilde{f}_L(\x) := \inprod{ \a }{ \tilde\x\ps{L+1}(\x) }, 
\end{equation}
where $\tilde\x\ps{1} = \x\ps{1} = \x$. 
It is clear that $\tilde{f}_L$ is a polynomial of degree $p_1 \cdots p_L$. 
Our goal is to show that we can choose $(R_l)_l, \eps_\phi$ properly, so that $\norm{ f_L - \tilde{f}_L }_{L^2}$ is small and the total degree $p_1 \cdots p_L$ is much smaller than $d^{1/2}$.

Consider an intermediate layer $l \in [L]$ and fix $q \ge 1$. 
For the approximation error at layer $l$, we have 
\begin{align*}
  \norm{ \x\ps{l+1} - \tilde\x\ps{l+1}  }_{L^q}
  &= \norm{ \bphi_l\left( \W\ps{l} \x\ps{l} \right) - \tilde\bphi_l\left( \W\ps{l} \tilde\x\ps{l} \right)  }_{L^q} \\
  &\le \norm{ 
      \bphi_l\left( \W\ps{l} \x\ps{l} \right) - \bphi_l\left( \W\ps{l} \tilde\x\ps{l} \right)  
    }_{L^q}
  + \norm{ 
    \bphi_l\left( \W\ps{l} \tilde\x\ps{l} \right) - \tilde\bphi_l\left( \W\ps{l} \tilde\x\ps{l} \right)  
  }_{L^q}. 
\end{align*}
By the Lipschitzness of $\bphi_l$ and the operator-norm bound on $\W\ps{l}$, we can bound the first error term using $\norm{\x\ps{l} - \tilde\x\ps{l}}_{L^q}$. 
Hence, to obtain a recurrence relation that controls the approximation error, it suffices to analyze the second error term. 

One can show that, for every $l \in [L]$, the exact intermediate representation $\x\ps{l}$ is subgaussian after centering (cf.~Lemma~\ref{lemma: lb: tail bounds for the intermediate representations}). Consequently, the corresponding preactivations lie in intervals of length at most $\poly\log(d)$ with high probability.
We seek an analogous statement for the approximate preactivations $(\w\ps{l}_k \cdot \tilde{\x}\ps{l})_k$.
Because $\tilde\bphi_l$ approximates $\bphi_l$ over the intervals $(I\ps{l}_k)_k$, a natural strategy is to choose intervals of length $\poly\log d$, show that the approximate preactivations lie in them with high probability, and control the contribution from the complement.

Unfortunately, the above strategy cannot be implemented in a straightforward global way. 
To approximate $\bphi_l$ on larger intervals using polynomials, generally we need polynomials of higher degrees (cf.~Appendix~\ref{sec: chebyshev} and Lemma~\ref{lemma: chebyshev approx}). 
Consequently, the contribution outside the intervals can also grow. 

To see why this prevents a direct global argument, suppose that the degree and interval length at every layer are $p$ and $R$, respectively. 
By Lemma~\ref{lemma: chebyshev approx}, this requires choosing $p \gtrsim R \log R \ge R$. 
Hence, the total degree will be at least $R^L$ and the $R^L$-th moment of a subgaussian variable can be as large as $R^{L R^L}$. 
On the other hand, the bound on the failure probability is at best $e^{-\Theta(R^2)}$. 
Hence, in order for $e^{-\Theta(R^2)} R^{L R^L} \ll 1$ to hold, we would need $L R^L \ll R^2$, which is impossible when $L \ge 2$. 

To fix the above issue, we will use a layerwise argument. 
Let $\cG_l$ denote the event that all preactivations at layer $l$ lie in the corresponding $I\ps{l}_k$.
Then, let $\cB_l$ denote the event that the first failure happens at layer $l$. 
That is, $\cB_l = \cG_1 \cap \cdots \cap \cG_{l-1} \cap \cG_l^c$.
Now, consider layer $l$ and event $\cB_l$. 
Since the preactivations in earlier layers all lie in their corresponding approximation interval, one can show that the layer-$l$ inputs are subgaussian (cf.~Lemma~\ref{lemma: lb: bad event probability, cond subg}).
As a function of the layer-$l$ input, $\tilde{f}_L$ is a polynomial of degree $D_l := \prod_{k=l}^L p_k = p_l D_{l+1}$. 
Hence, typical size of the final output is of order at most $d^{D_l} = d^{p_l D_{l+1}} \approx d^{R_l (\log R_l) D_{l+1}}$.
Meanwhile, the failure probability at this layer is roughly $e^{- \Theta(R_l^2)}$. 
Therefore, the constraint at this layer is $R_l^2 \gg R_l (\log R_l)(\log d) D_{l+1}$. 
Hence, when $R_l \le d$, it suffices to require $R_l \gg (\log d)^2 D_{l+1}$. 
Note that $D_{L+1} = 1$ and $D_{l+1}$ does not depend on $R_l$. 
Hence, we can solve the above constraints backward, from layer $L$ to layer $1$. 
At each layer, the degree roughly doubles. 
Because the number of layers is constant, both the degree and the interval length remain polylogarithmic.

\section{Upper bound for logarithmic-depth networks}
\label{sec: upper bound}

In this section, we first describe our learner architecture and training algorithm. 
Then, we formally state our upper bound for logarithmic-depth networks and sketch the proof. 
The complete proof can be found in Appendix~\ref{appendix: upper bound proofs}.

Throughout this section, $f_*: \{\pm 1\}^d \to \R$ will denote our target function and $\cS_*$ its Fourier spectrum. 
We assume $f_* \in \HS_{L, D, m, \kappa}$. 
That is, $f_*$ is a hierarchical staircase function with depth $L$, local degree $D$, width $m$, and signal strength $\kappa$. 
In addition, we assume $\norm{f_*}_{L^2} \le 1$.

Recall from the introduction that a hierarchical staircase function $f_*$ is defined to have a Fourier spectrum $\cS_*$ that can be recovered through the recurrence relation:
\[
  \cK_1 = \{ \varnothing, \{1\}, \dots, \{d\} \}, \quad 
  \cK_{l+1} 
  = \cK_l \cup \left(  \cK_l^{\oplus D} \cap \cS_*  \right). 
\]
This definition suggests the following conceptual algorithm.
At step $l$, we first form all products of at most $D$ learned features in $\cK_l$.
Instead of checking all high-degree features, we test only these candidates for membership in $\cS_*$ and discard those outside $\cS_*$.
We then add the remaining features to $\cK_l$.
The use of sets ensures we keep exactly one copy of each feature. 
The above pruning and deduplication operations ensure that at each step, we only check and store polynomially many features. 

In the following (and Appendix~\ref{appendix: upper bound proofs}), we show that the above conceptual algorithm can be implemented by training a deep network using layerwise coordinate descent.

\subsection{Setup and main result}

Our learner network is an $L$-layer network. 
For $l \in [L]$, let $\y\ps{l} \in \R^m$ denote the output of the $l$-th layer of the network.\footnote{
  Our analysis is still valid when the intermediate width of the learner is strictly larger than $m$. 
  We choose the intermediate width to be exactly $m$ to avoid introducing an unnecessary symbol for it. 
} 
In addition, we set $\y\ps{0} = \mbf{0} \in \R^m$. 
Let $M = \sum_{k=0}^{D} \binom{d + m}{k}$ denote the number of subsets of $[d + m]$ of size at most $D$. 
Fix an arbitrary ordering of $\binom{[d+m]}{\le D}$, so that we can identify $[M]$ with $\binom{[d+m]}{\le D}$ and index the entries of vectors in $\R^M$ using subsets of $[d+m]$.
Define the degree-$D$ polynomial feature mapping $\bPhi$ to be
\begin{equation}
  \label{eq: poly feature mapping}
  \bPhi: \R^{d + m} \to \R^M, \quad 
  \x' \mapsto \left( \chi_H(\x') \right)_{H \subset [d + m], |H| \le D}, 
\end{equation}
where $\chi_H(\x') := \prod_{i \in H} x_i'$ is the multilinear monomial associated with the set $H$. 
Let $\odot$ and $\circ$ denote entrywise multiplication and vector concatenation, respectively. 
Then, we define our learner network inductively as 
\begin{equation}
  \label{eq: learner network}
  \begin{aligned}
    & \text{$l$-th layer input:}
    && \x\ps{l}
    := \x \circ \y\ps{l-1} \in \R^{d + m}, \quad \\
    & \text{$l$-th layer output:}
    && \y\ps{l}
    := \Select_{\w\ps{l}}\left( \w\ps{l} \odot \bPhi\big( \x\ps{l} \big) \right) 
    =: \Select_{\w\ps{l}}\left( \w\ps{l} \odot \z\ps{l} \right) \in \R^m, \\
    & \text{model output:}
    && f(\x)
    := \inprod{\w\ps{L}}{\z\ps{L}} \in \R,
  \end{aligned}
\end{equation}
where $l \in [L]$, $(\w\ps{l})_{l \in [L]}$ are the trainable weights and $\Select: \R^M \to \R^m$ is defined as 
\[
  \Select_{\w}(\w \odot \z) := ( z_j' )_{ j \in [m] } 
  \quad\text{where}\quad 
  \z' := ( w_i z_i )_{i \in [M] \,:\, w_i \ne 0} \circ (0, 0, \dots)
  \in \R^{\bbN_+}.
\]
In NumPy/PyTorch-style notation, $\Select$ can be defined more compactly as 
\[
  \Select_{\w}(\w \odot \z) 
  := \texttt{PadZero}\left( (\w \odot \z)[ \w \ne 0 ], \texttt{length}=m \right)[:m].
\]
Note that after training, if $w\ps{l}_i = 0$, then $w\ps{l}_i z\ps{l}_i$ is identically zero. 
In other words, $\Select$ keeps the first $m$ nonzero coordinates.  
In addition, we define $f\ps{l}(\x) := \inprod{\w\ps{l}}{\z\ps{l}}$ for all $l \in [L]$. 
This can be interpreted as the scalar readout from the $l$-th layer. 
We also write $\x\ps{L+1} := \x \circ \x\ps{L}$.

\begin{algorithm}[t]
  \caption{Training algorithm (layerwise one-pass coordinate descent)}
  \label{alg: training alg}
  \begin{algorithmic}[1]
    \State \tnbf{Algorithm parameters:} Threshold $\lambda_w > 0$, number of samples per layer $N$, step size~$\eta$, number of steps per coordinate $T$.
    \State \tnbf{Initialization:} Set $\w\ps{l} = 0$ for all $l \in [L]$. 
    \For{$l \in [L]$} \Comment{layerwise train each layer} \label{line: layerwise for loop}
      \State Gather $N$ fresh i.i.d.~samples and define $\Loss\ps{l}$ as in \eqref{eq: training loss}.
      \For{$H \in [M]$} \Comment{one-pass coordinate descent} \label{line: coordinate descent for loop}
        \State 
          Run coordinate descent \eqref{eq: coordinate descent} with respect to $w\ps{l}_H$ for $T$ steps with step size $\eta$. 
          \label{line: coordinate descent}
        \State $w\ps{l}_H \gets 0$ if $|w\ps{l}_H| \le \lambda_w$. 
          \label{line: truncation}
      \EndFor
    \EndFor
    \State \Return $(\w\ps{l})_{l \in [L]}$.
  \end{algorithmic}
\end{algorithm}

Now, we describe our training algorithm. 
The pseudocode is given in Algorithm~\ref{alg: training alg}. 
In short, we layerwise train each layer of \eqref{eq: learner network} using one-pass coordinate descent.
Namely, the training procedure consists of $L$ stages. 
In stage~$l \in [L]$, we sample $N$ independent input-output pairs.
For notational simplicity, let $\hat{\E}\ps{l}$ denote the empirical expectation over this dataset. 
We choose the loss for the $l$-th layer to be
\begin{equation}
  \label{eq: training loss}
  \begin{aligned}
    \Loss\ps{l}(\w\ps{l})
    &:= \frac{1}{2} \hat{\E}\ps{l} \left( f_*(\x) - f\ps{l}(\x)  \right)^2 
    = \frac{1}{2} \hat{\E}\ps{l} \left( f_*(\x) - \inprod{\w\ps{l}}{\z\ps{l}}  \right)^2. 
  \end{aligned}
\end{equation}
We train the $l$th layer weight $\w\ps{l}$ by sequentially minimizing $\Loss\ps{l}$ with respect to each coordinate using gradient descent. 
That is, we initialize $\w\ps{l}$ to be $0$, and sequentially set each of its coordinates $H \in \binom{[d+m]}{\le D} \cong [M]$ by minimizing the one-dimensional (quadratic) function $w\ps{l}_H \mapsto \Loss\ps{l}(\w\ps{l})$ using gradient/coordinate descent:
\begin{equation}
  \label{eq: coordinate descent}
  w\ps{l}_H(t+1) 
  = w\ps{l}_H(t) 
    - \eta \partial_{w\ps{l}_H} \Loss\ps{l}\left( \w\ps{l}_{-H} + w\ps{l}_H(t) \e_H \right), \quad 
  t \in [T], 
\end{equation}
where $\eta > 0$ is the step size, $T$ is the number of steps, and $\w_{-H}$ denotes the vector obtained by setting the $H$-th entry of $\w$ to be $0$.
Note that $\w_{-H}\ps{l}$ does not change when we update~$w\ps{l}_H$.

\begin{theorem}
  \label{thm: upper bound}
  Let $\eps_*, \delta_{\P} \in (0, 0.1)$ denote the target accuracy and failure probability, respectively.
  Let $f_* \in \HS_{L, D, m, \kappa}$ be our target function, where $D$ is a universal constant.
  Suppose that $\norm{f_*}_{L^2} \le 1$ and choose the parameters of Algorithm~\ref{alg: training alg} to be 
  \begin{gather*}
    \lambda_w = 2^{-D-1} \kappa, \quad 
    N 
    \gtrsim \frac{
      M^2 \log(L  M / \delta_{\P} )
    }{
      \kappa^{2D+2} 
      \left( \eps_*^2 / m \wedge \kappa^{2D+2} \right) 
    } , \quad 
    \eta
    = 2^{-2D},  \quad 
    T 
    \gtrsim \frac{
      \log\left(m/ \eps_* \vee 1/\kappa \right)
    }{ 
      \kappa^{2D} \eta
    },
  \end{gather*}
  for every $l \in [L]$, then with probability at least $1 - \delta_{\P}$, Algorithm~\ref{alg: training alg} will output a network $f$ satisfying $\norm{ f_* - f }_{L^2} \le \eps_*$.
\end{theorem}
\begin{remark}
  Note that each layer of our learner network \eqref{eq: learner network} is at most a degree-$D$ polynomial. 
  Hence, when the learner depth is constant, \eqref{eq: learner network} cannot learn the high-degree part of the target. 
  In addition, we show in Appendix~\ref{appendix: approx ub net with lb net} that the trained network can be efficiently compiled into a network matching the architecture of our lower bound. 
\end{remark}

\subsection{Proof sketch}

Suppose that, for every $S \in \cS_*$, there exist an index $k$ and a nonzero scalar $\alpha_k$ such that the corresponding entry of $\z\ps{L} := \bPhi(\x\ps{L})$ satisfies $z\ps{L}_k(\x) = \alpha_k \chi_S(\x)$ for all $\x \in \{\pm 1\}^d$. Set the corresponding output weight to $w\ps{L}_k = \hat{f}_*(S) / \alpha_k$ for each $S \in \cS_*$ and set all remaining output weights to zero. Then $f(\x) = f_*(\x)$.
Hence, our goal is to show that after training, $\z\ps{L}$ will encode all the sets in the Fourier spectrum $\cS_*$. 

To this end, we will show by induction on $l \in [L]$ that $\x\ps{l}$ contains all features present in $\cK_l \setminus \{\varnothing \}$.\footnote{We do not need to track $\varnothing$ (which corresponds to the constant term in the Fourier decomposition of $f_*$), since the $\varnothing$-th entry in \eqref{eq: poly feature mapping} always gives the constant term.} 
Formally, we maintain the following induction hypothesis. 
\begin{restatable}[Induction hypothesis on the input]{assumption}{UpperBoundIH}
  \label{assumption: training induction hypothesis}
  Fix $l \in [L]$. 
  We assume as an induction hypothesis the following properties of the $l$-th-layer input $\x\ps{l} := (\x, \y\ps{l-1}) \in \R^{d + m}$. 
  \begin{enumerate}[(a)]
    \item \label{itm: training IH: zero, nonzero}
      For every $k \in [d + m]$, $x\ps{l}_k$, as a function of the input $\x$, is either identically zero or nowhere zero. 
      We write $I_l := \big\{ k \in [d+m] \,:\, x\ps{l}_k \not\equiv 0 \big\}$ for the collection of nonzero coordinates.
    \item \label{itm: training IH: features}
      There exist nonzero scalars $( \alpha\ps{l}_k )_{k \in I_l}$ and subsets $( S\ps{l}_k )_{k \in I_l}$ of $[d]$ such that $x\ps{l}_k = \alpha\ps{l}_k \chi_{S\ps{l}_k}(\x)$ with $\alpha\ps{l}_k \in \{1\} \cup (1 \pm 0.2) \hat{f}_*( S\ps{l}_k )$ for every $k \in I_l$.
      We call $\cF_l := \big\{ S\ps{l}_k \big\}_{k \in I_l}$ the features contained in $\x\ps{l}$. 
      Moreover, for any $S \in \cF_l$ with $|S| \ge 2$, there is exactly one $k \in I_l$ with $S\ps{l}_k = S$.

    \item \label{itm: training IH: covering Kl}
      $\x\ps{l}$ contains all nonempty features in $\cK_l$, i.e., $\cF_l \supset \cK_l \setminus \{\varnothing \}$.
  \end{enumerate}
\end{restatable}

The base case follows directly from $\x\ps{1} := \x \circ \mbf{0}$.
For the induction step, first, note that if $H \in \binom{[d + m]}{\le D}$ contains an index that is not in $I_l$, then $z\ps{l}_H$ is identically zero. 
As a consequence, $w\ps{l}_H$ will stay at its initial value $0$. 
Hence, it suffices to consider those entries with $H \subset I_l$. 
For notational simplicity, we define 
\begin{equation}
  \label{eq: Hl}
  \cH_l := \braces{ H \subset I_l \,:\, |H| \le D }.
\end{equation}
This is the collection of indices of the nonzero entries of $\z\ps{l}$.
By Assumption~\ref{assumption: training induction hypothesis}, for every $H \in \cH_l$, we have 
\begin{equation}
  \textstyle
  \label{eq: zlH}
  z\ps{l}_H 
  = \prod_{i \in H} \alpha\ps{l}_i \chi_{S\ps{l}_i}(\x)
  = \left( \prod_{i \in H} \alpha\ps{l}_i \right) \chi_{\bigoplus_{i \in H} S\ps{l}_i}(\x)
  =: \alpha\ps{l}_H \chi_{S\ps{l}_H}(\x),
\end{equation}
where $\alpha\ps{l}_H :=  \prod_{i \in H} \alpha\ps{l}_i \ne 0$ and $S\ps{l}_H := \bigoplus_{i \in H} S\ps{l}_i$. 
We write $\cF_{l+1/2} := \big\{ S\ps{l}_H \big\}_{H \in \cH_l }$ and call these subsets the features contained in $\z\ps{l}$. 
By the construction of $(\cK_l)_l$ in \eqref{eq: def of cK}, $\cF_{l+1/2} \supset \cK_{l+1}$. 

It remains to select the features that are actually in $\cK_{l+1}$ and remove duplicate representations.
Deduplication is important, since duplicate representations will create more duplicate representations in later layers and quickly saturate the model capacity. 
We show that both feature selection and deduplication can be implemented simultaneously using one-pass coordinate descent (Algorithm~\ref{alg: training alg}). 

The complete formal proof can be found in Appendix~\ref{appendix: upper bound proofs}. 
To see this intuitively, suppose that $S\ps{l}_H = S\ps{l}_{H'} = S$ for some $H \ne H' \in \cH_l$. 
The only terms involving $H$ and $H'$ in $\inprod{\w\ps{l}}{\z\ps{l}}$ are
\[
  w\ps{l}_H z\ps{l}_H + w\ps{l}_{H'} z\ps{l}_{H'}
  = \left( w\ps{l}_H \alpha\ps{l}_H + w\ps{l}_{H'} \alpha\ps{l}_{H'} \right) \chi_S(\x),
\]
where the identity comes from \eqref{eq: zlH}. 
In particular, only $w\ps{l}_H \alpha\ps{l}_H + w\ps{l}_{H'} \alpha\ps{l}_{H'}$ matters. 
If we encounter $H$ first (so $w\ps{l}_{H'}$ is still $0$) and fit its coefficient exactly so that $w\ps{l}_H \alpha\ps{l}_H \chi_S(\x) = \hat{f}_*(S) \chi_S(\x)$, then, when we reach $H'$, the residual will be $0$ and, therefore, $w\ps{l}_{H'}$ will stay at its initial value $0$.

\section{The depth-separation theorem}
\label{sec: sep thm}

In this section, we combine the results from the previous two sections and prove the separation theorem between constant- and logarithmic-depth networks. 
To this end, we first formally define the class of deep quadratic functions.

\begin{definition}[Deep quadratic function]
  \label{def: deep quad function}
  Suppose $d = 2^{L_*}$ for some $L_* \in \bbN_+$ and let $d_l := 2^{-l} d$. 
  We say a function $f_*: \{\pm 1\}^d \to \R$ is a deep quadratic function if it can be defined as 
  \[
    f_*(\x) 
    = L_*^{-1/2} \sum_{l=1}^{L_*} 2^{-(L_*-l)/2} \inprod{\One}{\x\ps{l}}, \quad  
    \x\ps{l}
    = \left( x\ps{l-1}_{\sigma_l(2k-1)} x\ps{l-1}_{\sigma_l(2k)} \right)_{k \in [d_l]} \in \R^{d_l}, \quad 
    \forall l \in [L_*], 
  \]
  where $\x\ps{0} = \x \in \R^{d_0}$ and each $\sigma_l$ is an arbitrary permutation of $[d_{l-1}]$. 
\end{definition}

\begin{fact}
  \label{fact: deep quad function}
  Suppose that $d = 2^{L_*}$.
  Every deep quadratic function $f_*: \{\pm 1\}^d \to \R$ is a hierarchical staircase function with depth $L_*+1$, local degree $2$, width $d-1$, and signal strength $\sqrt{2 / (L_* d)}$. 
  In addition, $\norm{ f_* }_{L^2} = 1$ and $\sum_{|S| \ge \sqrt{d} } \hat{f}_*^2(S) \ge 1/3$.  
\end{fact}

By combining Theorem~\ref{thm: lower bound against constant-depth networks}, Theorem~\ref{thm: upper bound}, Fact~\ref{fact: deep quad function}, and Corollary~\ref{cor: compile: deep quad}, we obtain the following theorem.

\begin{theorem}[Depth separation]
  \label{thm: sep}
  Suppose that $d = 2^{L_*}$ and that $d$ is sufficiently large.
  \begin{enumerate}[(a)]
    \item For any deep quadratic function $f_*: \{\pm 1\}^d \to \R$ and any constant-depth network satisfying Assumption~\ref{assumption: lower bound assumption}, we have $\norm{ f_L - f_* }_{L^2(\Unif(\{\pm 1\}^d))} \ge 1/4$.
    \item Let $\eps_*, \delta_{\P} \in (0, o(1))$ denote the target accuracy and failure probability, respectively.
    Consider the depth-$(L_*+1)$ learner model in \eqref{eq: learner network}, with the model and algorithm parameters chosen according to Theorem~\ref{thm: upper bound} and Fact~\ref{fact: deep quad function}.
    Then, for any target deep quadratic function $f_*$, Algorithm~\ref{alg: training alg} will output a network $f$ satisfying $\norm{ f_* - f }_{L^2} \le \eps_*$ with probability at least $1 - \delta_{\P}$.

    Moreover, we can efficiently compile $f$ into a depth-$(L_*+3)$ network satisfying all conditions in Assumption~\ref{assumption: lower bound assumption} except for the condition $L = O(1)$.
  \end{enumerate}
\end{theorem}

\section{Conclusion}
\label{sec: conclusion}

In this work, we establish the first rigorous algorithmic separation between logarithmic-depth and constant-depth neural networks.
We show that a depth-$L$ network trained by layerwise coordinate descent efficiently learns depth-$L$ hierarchical staircase functions by adaptively composing previously discovered Fourier features.
In contrast, constant-depth networks satisfying our regularity and norm assumptions admit low-degree polynomial approximations and therefore cannot approximate deep quadratic functions, which place constant Fourier mass at high degrees.
These results identify adaptive hierarchical feature construction as a mechanism through which depth enables efficient learning.
An important direction is to extend this separation to end-to-end gradient-based training and more standard architectures while relaxing the assumptions required by the lower bound.

\section*{Acknowledgements}
The authors would like to thank Theodor Misiakiewicz and Kaifeng Lyu for helpful discussion
and feedback.
YR and JDL acknowledges support of  NSF IIS 2107304,  NSF CCF 2212262, NSF CAREER Award 2540142, NSF 2546544, NSF CCF 2019844 and ONR N00014-24-1-2639.

\section*{Statement of AI Use}
We used GPT 5.4/5.5/5.6 Pro to perform preliminary calculations and sanity checks while developing the proof strategy, to proofread the final manuscript, and to assist with several routine calculations.
All arguments and calculations presented in the paper were subsequently derived and verified independently by the authors.

\bibliography{reference}
\clearpage

\appendix
{
  \hypersetup{linkcolor=black}
  \tableofcontents
}

\section{Chebyshev polynomials and approximation}
\label{sec: chebyshev}

\begin{definition}[Chebyshev polynomials]
  The \tnbf{Chebyshev polynomials (of the first kind)} $(T_n)_{n \in \bbN_0}$ are the polynomials with real coefficients defined by  
  \[
    T_n(x) 
    = \begin{cases}
      \cos( n \arccos x ),            & |x| \le 1, \\
      \cosh( n \arccosh x ),          & x \ge 1 \\
      (-1)^n \cosh( n \arccosh(-x) ), & x \le -1.
    \end{cases}
  \]
\end{definition}
\begin{remark}
  Equivalently, Chebyshev polynomials can be defined via the recurrence relation 
  $T_0(x) = 1$, $T_1(x) = x$, $T_{n+1}(x) = 2 x T_n(x) - T_{n-1}(x)$, which makes it obvious that $T_n$ is a degree-$n$ polynomial. 
\end{remark}

\begin{fact}
  \label{fact: chebyshev: bounds for Tn}
  For any $n \in \bbN_0$ and $x \in \R$, we have $\abs{ T_n(x) } \le (1 + 2|x|)^n$. 
\end{fact}
\begin{proof}
  For $x \in [-1, 1]$, it is clear that $|T_n(x)| \le 1$. 
  When $|x| \ge 1$, we have 
  \begin{align*}
    |T_n(x)|
    = \cosh( n \arccosh |x| )
    &\le \exp\left( n \arccosh|x| \right)  \\
    &= \exp\left( n \log\left( |x| + \sqrt{x^2 - 1} \right) \right) 
    \le \exp\left( n \log\left( 2 |x| \right) \right) 
    = ( 2|x| )^n.
  \end{align*}
\end{proof}

\begin{lemma}[Equation (2.37) of \cite{rivlin_chebyshev_1990}]
  \label{lemma: poly <= Tn}
  For any degree-$n$ polynomial $f: \R \to \R$ with $\norm{ f }_{L^\infty([-1, 1])} \le 1$, we have $|f(x)| \le |T_n(x)|$ for all $|x| \ge 1$. 
  Combining this with Fact~\ref{fact: chebyshev: bounds for Tn}, we obtain $|f(x)| \le (1 + 2|x|)^n$. 
\end{lemma}

\begin{definition}[Bernstein ellipse]
  Let $\rho > 1$. 
  The \tnbf{Bernstein ellipse} with parameter $\rho$ is defined as 
  \[
    E_\rho := \braces{ ( w + w\inv ) / 2 \,:\, w \in \C, |w| = \rho  } \subset \C. 
  \]
  In other words, $E_\rho$ is the ellipse in $\C$ with foci $\pm 1$ and semi-major and semi-minor axes $(\rho \pm \rho\inv)/2$.
  We use $E_\rho^o$ to denote the open region enclosed by $E_\rho$.
\end{definition}

\begin{lemma}[Theorem 8.2 of \cite{trefethen_approximation_2019}]
  \label{lemma: chebyshev approx}
  Let $\rho > 1$. 
  Suppose that $f: \R \to \R$ is real analytic, can be analytically continued to $E_\rho$, and is bounded by $M$ in $E_\rho$. 
  Then, its Chebyshev coefficients $(a_p)_p$ satisfy 
  \[
    \abs{ a_p } \le 2 M \rho^{-p}, \quad \forall p \in \bbN.
  \]
  In particular, this implies that its degree-$p$ Chebyshev approximation $\tilde{f}$ satisfies 
  \[
    \norm{ f - \tilde{f} }_{L^\infty([-1, 1])}
    \le \frac{2 M}{\rho - 1} \rho^{-p}.
  \]
\end{lemma}

\begin{lemma}
  \label{lemma: chebyshev approx: strip}
  Consider $\phi: \R \to \R$. 
  Suppose that for some $c_0 > 0$, $\phi$ can be analytically continued to the strip $\braces{ z \in \C \,:\, \abs{ \Im\,z } < c_0 }$, and $\phi$ has at most polynomial growth:
  \[
    \abs{ \phi(a + i b) } \le K ( 1 + |a| )^m,  \quad 
    a, b \in \R, \, |b| \le c_0 / 2,
  \]
  for some $K, m \ge 0$. 

  Fix $x_0 \in \R$ and $R > 0$ and define $\psi(u) := \phi(x_0 + R u)$.
  Let $\psi(u) = \sum_{n=0}^\infty a_n T_n(u)$ be the Chebyshev expansion of $\psi$ on $[-1, 1]$ and define $\tilde{\psi}_p := \sum_{n=0}^p a_n T_n$. 
  We have $|a_n| \le 2 M_R \rho_R^{-n}$ for every $n \in \bbN$, and 
  \[
    \norm{ \psi - \tilde{\psi}_p }_{L^\infty([-1, 1])} 
    \le \frac{2 M_R}{\rho_R - 1} \rho_R^{-p}, \quad 
  \] 
  where 
  \[
    \rho_R := 1 + \frac{c_0}{2 R}, \quad 
    M_R := K \left( 1 + |x_0| + R + c_0 / 2 \right)^m.
  \]
  In particular, for any target accuracy $\eps > 0$, if
  \[
    R \ge \frac{c_0}{2}, \quad 
    p \ge \frac{4 R}{c_0}\log\left( \frac{4 R M_R}{c_0 \eps } \right),
  \]
  the degree-$p$ polynomial $\tilde{\phi}(x) := \tilde{\psi}_p( ( x - x_0) / R )$ satisfies
  \[
    \norm{ \phi - \tilde{\phi}  }_{L^\infty([x_0 - R, x_0 + R])}
    = \norm{ \psi - \tilde{\psi}_p  }_{L^\infty([-1, 1])}
    \le \eps.
  \]
\end{lemma}
\begin{proof}
  First, we show that $z \mapsto \psi(z) = \phi(x_0 + R z)$ is analytic in $E_{\rho_R}$. 
  For any $z \in E_{\rho_R}^o$, since the semi-minor axis of $E_{\rho_R}$ is $( \rho_R - \rho_R\inv ) / 2$, we have 
  \[
    \abs{ \Im( x_0 + R z ) }
    = R \abs{ \Im \, z }
    \le R ( \rho_R - \rho_R\inv ) / 2
    \le c_0 / 2.
  \]
  Since $\phi$ is analytic in the strip of half-width $c_0$, this implies $\psi$ is analytic in $E_{\rho_R}$. 
  To apply Lemma~\ref{lemma: chebyshev approx}, we also need to bound the maximum magnitude of $\psi$ in $E_{\rho_R}$.
  Since the semi-major axis of $E_{\rho_R}$ is $(\rho_R + \rho_R\inv)/2 \le \rho_R = 1 + c_0 / (2R)$, for any $z \in E_{\rho_R}^o$, we have 
  \[
    \abs{ \Re(x_0 + R z) }
    \le |x_0| + R \abs{\Re\, z}
    \le |x_0| + R (1 + c_0 / (2R))
    = |x_0| + R + c_0/2.
  \]
  Hence, by the polynomial growth of $\phi$, we have 
  \[
    \abs{ \psi(z) }
    = \abs{ \phi(x_0 + R z) }
    \le K ( 1 + |x_0| + R + c_0 / 2 )^m =: M_R, \quad 
    \forall z \in E_{\rho_R}^o.
  \]
  Thus, by Lemma~\ref{lemma: chebyshev approx}, we have 
  \[
    |a_n| \le 2 M_R \rho_R^{-n}, \quad 
    \norm{ \psi - \tilde{\psi}_p }_{L^\infty([-1, 1])} 
    \le \frac{2 M_R}{\rho_R - 1} \rho_R^{-p}. 
  \]
  For the second part of the lemma, note that when $R \ge c_0 / 2$, we have $c_0 / (2 R) \le 1$ and therefore, $1 + \rho_R \le 3$ and $\log \rho_R = \log(1 + c_0/(2R)) \ge c_0 / (4R)$.
  As a result, 
  \begin{align*}
    \frac{2 M_R}{\rho_R - 1} \rho_R^{-p}
    \le \eps 
    \quad\Leftarrow\quad 
    p \log \rho_R
    \ge \log\left( \frac{4 R M_R}{c_0 \eps } \right) 
    \quad\Leftarrow\quad 
    p \ge \frac{4 R}{c_0}\log\left( \frac{4 R M_R}{c_0 \eps } \right).
  \end{align*}
\end{proof}

\begin{lemma}[Common activation functions]
  \label{lemma: chebyshev approx: common activations}
  The following activation functions satisfy the assumptions of Lemma~\ref{lemma: chebyshev approx: strip} with $c_0 = \pi / 2$, $K \le 4$ and $m \le 1$: 
  \[
    \tanh x,\quad
    \mrm{Sigmoid}(x) := \frac{1}{1+e^{-x}},\quad
    \mrm{Softplus}(x) := \log(1+e^x),\quad
    \mrm{GELU}(x) := x \Phi(x),
  \]
  where $\Phi$ is the standard Gaussian CDF. 
\end{lemma}
\begin{proof}
  In this proof, we call $\{ z \in \C \,:\, |\Im\, z| < c_0  \}$ the $c_0$-strip. 

  The function $\tanh z$ is meromorphic with poles at $i\pi(k+1/2)$, $k \in \mbb{Z}$, and is therefore analytic in the $\pi/2$-strip.
  Moreover, for $|b| \le \pi/4$,
  \[
    |\tanh(a+ib)|^2
    = \frac{ \sin^2(2b) + \sinh^2(2 a) }{ \left( \cos(2 b) + \cosh(2 a)  \right)^2 }
    \le \frac{ 1 + \sinh^2(2 a) }{ \cosh^2(2 a) }
    = \frac{ 1}{ \cosh^2(2 a) } + \tanh^2(2 a) 
    \le 2, 
  \]
  so the strip-growth condition holds with $K = 2$ and $m = 0$.

  Sigmoid has possible poles only where $1+e^{-z}=0$, namely at $(2k+1)i\pi$, $k \in \mbb{Z}$.
  Thus, it is analytic in the $\pi$-strip. 
  If $|b|\le \pi/2$, then $\Re(e^{-a-ib}) = e^{-a}\cos b \ge 0$, and hence $|1 + e^{-a-ib}| \ge 1$.
  Therefore, $| \mrm{Sigmoid}(a + i b) | \le 1$. 
  Hence, sigmoid satisfies the strip-growth condition with $K = 1$ and $m = 0$. 

  Now consider softplus.
  Again, since the only zeros of $1 + e^z$ are at $(2 k + 1 ) i \pi$, $k \in \mbb{Z}$, $\log(1 + e^z)$ is analytic in the $\pi$-strip.
  Hence, we can take the branch of $\log(1 + e^z)$ that matches the real $\log(1 + e^x)$ on the real line. 
  Now, consider $z = a + i b$ with $a, b \in \R, |b| \le \pi/2$. 
  As in the preceding argument, $\Re(1 + e^z) \ge 1$. 
  Therefore, we have 
  \[
    \abs{ \Re \log(1 + e^z) }
    = \log \abs{ 1 + e^z }
    \le \log (1 + e^a)
    \le \log 2 + |a|. 
  \]
  For the imaginary part, we have 
  \[
    \abs{ \Im \log(1 + e^z) }
    = \abs{ \mrm{Arg}\, (1 + e^z) }
    \le |b|
    \le \pi/2.
  \]
  As a result, on the $\pi/2$-strip, we have 
  \[
    \abs{ \log(1 + e^z) }
    \le \sqrt{ (\log 2 + |a|)^2 + (\pi/2)^2 }
    \le \pi/2 + \log 2 + |a|.
  \]
  Hence, softplus satisfies the strip-growth condition with $K = 4$ and $m = 1$. 

  Finally, consider the GELU activation $\mrm{GELU}(x)=x\Phi(x)$.
  Since $w \mapsto e^{-w^2/2}$ is entire, so is $\Phi(z) = 1/2 + (2\pi)^{-1/2} \int_0^z e^{-w^2/2} \,\rd w$. 
  Hence, $\mrm{GELU}$ can be analytically continued to any strip. 
  For simplicity, take the $\pi$-strip. 
  For $z = a + i b$ with $a, b \in \R, |b| \le \pi/4$, to bound $\Phi(z)$, we first integrate $e^{-w^2/2}$ from $0$ to $a$ and then from $a$ to $a + i b$. 
  This gives 
  \begin{align*}
    \abs{ \Phi(z) }
    &\le \frac{1}{2} 
      + \frac{1}{\sqrt{2\pi}} \int_0^{|a|} e^{-t^2/2} \,\rd t 
      + \frac{1}{\sqrt{2\pi}} \int_0^{|b|} \abs{ e^{ -(a + i s)^2 / 2 } } \,\rd s  \\
    &\le 1 + \frac{1}{\sqrt{2\pi}} \int_0^{\pi/2} e^{s^2/2} \,\rd s 
    \le 4.
  \end{align*}
  As a result, we have 
  \[
    \abs{ \mrm{GELU}(z) }
    \le 4 |z| 
    \le 4 \left( 1 + |a|  \right).
  \]
  Thus, it satisfies the strip-growth condition with $K = 4$ and $m = 1$.
\end{proof}

\section{Lower bound proofs}
\label{appendix: lower bound proofs}

In this section, we prove Theorem~\ref{thm: lower bound against constant-depth networks}. 
As discussed in the proof sketch, it suffices to show that any network satisfying Assumption~\ref{assumption: lower bound assumption} can be approximated by a polynomial of degree at most $d^{1/2}$.
Formally, we will prove the following lemma and then derive Theorem~\ref{thm: lower bound against constant-depth networks} as a corollary of it. 

\LBApproxFwithPoly*

The rest of this section is organized as follows. 
First, we show that the true intermediate representations $\x\ps{l}$ are subgaussian in Appendix~\ref{subsec: lb: true intermediate repr}. 
Then, in Appendix~\ref{subsec: lb: approx intermediate repr}, we prove Lemma~\ref{lemma: lb: approx f with poly} by approximating each layer of the learner network with a polynomial. 
Finally, in Appendix~\ref{subsec: lb: proof of the main theorem}, we prove Theorem~\ref{thm: lower bound against constant-depth networks}.

\subsection{Tail bounds for the intermediate representations}
\label{subsec: lb: true intermediate repr}

In this subsection, we show that the $\x\ps{l}$ are subgaussian.
The Gaussian Lipschitz-concentration argument does not apply here because vectors such as $\W\ps{1}\x\ps{1}$ are subgaussian but need not be Gaussian.
Instead, we use the following adaptive version of McDiarmid's inequality.

\begin{lemma}[Theorem 6.5 of \cite{boucheron_concentration_2013}]
  \label{lemma: bounded difference inequality}
  Let $\cX$ be the domain of each input coordinate. 
  Suppose that $f: \cX^d \to \R$ satisfies the following $\x$-dependent bounded-difference property: 
  there exist $c_1, \dots, c_d: \cX^{d-1} \to [0, \infty)$ such that for any $\x \in \cX^d$, 
  \[
    \sup_{y, y'\in \cX}
      \abs{ 
        f(x_1, \dots, x_{i-1}, y, x_{i+1}, \dots, x_d) 
        - f(x_1, \dots, x_{i-1}, y', x_{i+1}, \dots, x_d) 
      } 
    \le c_i( \x_{-i} ),
    \quad \forall i \in [d], 
  \]
  where $\x_{-i}$ denotes the vector obtained by removing the $i$-th coordinate of $\x$. 
  If $(1/4) \sum_{i=1}^{d} c_i^2(\x_{-i}) \le v$ for some $v > 0$, then we have 
  \[
    \P_{\x}\left[ \abs{ f(\x) - \E f(\x) } \ge t \right]
    \le 2 e^{ -t^2 / (2v) }, \quad 
    \forall t > 0,
  \]
  where $\x \in \cX^d$ is any random vector with independent entries. 
\end{lemma}
\begin{remark}
  When the domain is $\{\pm 1\}^d$, since $x_i \in \{\pm 1\}$, we can always choose 
  \[
    c_i(\x_{-i}) 
    := \abs{ 
        f(x_1, \dots, x_{i-1}, 1, x_{i+1}, \dots, x_d) 
        - f(x_1, \dots, x_{i-1}, -1, x_{i+1}, \dots, x_d) 
      }
    = \abs{ 
        f(\x) 
        - f(\x_{\oplus i}) 
      }, 
  \]
  where $\x \in \{\pm 1\}^d$ is any vector consistent with $\x_{-i}$ and $\x_{\oplus i}$ denotes the vector obtained by flipping the $i$-th coordinate of $\x$. 
  Therefore, it suffices to upper bound 
  \[
    \sup_{\x \in \{\pm 1\}^d} 
    \frac{1}{4} \sum_{i=1}^{d} \left( f(\x) - f(\x_{\oplus i}) \right)^2
    =: \sup_{\x \in \{\pm 1\}^d} \frac{1}{4} V_f(\x).
  \] 
\end{remark}

To illustrate how this lemma will be used, consider $f_l(\x) := \inprod{\a\ps{l}}{\x\ps{l+1}(\x)}$. 
Write 
\begin{align*}
  V_{f_l}(\x)
  = \sum_{k=1}^d \inprod{\a\ps{l}}{\x\ps{l+1}(\x) - \x\ps{l+1}(\x_{\oplus k}) }^2  
  &= \norm{
      (\a\ps{l})\trans
      \left[ 
        \x\ps{l+1}(\x)
        - \x\ps{l+1}(\x_{\oplus k})
      \right]_{k \in [d]}
    }^2 \\
  &\le \norm{ \a\ps{l} }_2^2
    \norm{ \left[ \x\ps{l+1}(\x) - \x\ps{l+1}(\x_{\oplus k}) \right]_{k \in [d]} }_{\OP}^2. 
\end{align*}
For the second factor, recall that $\x\ps{l+1}(\x) = \bphi_l(\W\ps{l}\x\ps{l}(\x))$. 
Hence, by Taylor expanding $\phi_{l, k}$ at $\inprod{\w\ps{l}_k}{\x\ps{l}}$, we get 
\begin{align*}
  \x\ps{l+1}(\x) - \x\ps{l+1}(\x_{\oplus k}) 
  &= \bphi_l\left( \W\ps{l}\x\ps{l}(\x) \right) 
    - \bphi_l\left( 
      \W\ps{l}\x\ps{l}(\x) 
      + \W\ps{l}\left( \x\ps{l}(\x_{\oplus k}) - \x\ps{l}(\x) \right) 
    \right)  \\ 
  &\approx 
    \diag\left( \bphi_l'\left( \W\ps{l}\x\ps{l}(\x) \right) \right) 
    \W\ps{l}\left( \x\ps{l}(\x_{\oplus k}) - \x\ps{l}(\x) \right), 
\end{align*}
where $\bphi_l' = ( \phi_{l, k}' )_k$. 
As a result, we have 
\begin{align*}
  \norm{ \left[ \x\ps{l+1}(\x) - \x\ps{l+1}(\x_{\oplus k}) \right]_{k \in [d]} }_{\OP}  
  &\lesssim \norm{
      \diag\left( \bphi_l'\left( \W\ps{l}\x\ps{l}(\x) \right) \right) 
      \W\ps{l}
      \left[ \x\ps{l}(\x_{\oplus k}) - \x\ps{l}(\x) \right]_{k \in [d]}
    }_{\OP}  \\
  &\le C_\phi C_W 
    \norm{ \left[ \x\ps{l}(\x_{\oplus k}) - \x\ps{l}(\x) \right]_{k \in [d]} }_{\OP}.
\end{align*}
Note that the last factor is exactly the left-hand side with $l+1$ replaced by $l$. 
Thus, we can iterate the above argument to  obtain an upper bound on the operator norm of $\left[ \x\ps{l+1}(\x) - \x\ps{l+1}(\x_{\oplus k}) \right]_{k \in [d]}$.
Combining this with Lemma~\ref{lemma: bounded difference inequality}, we conclude that $f\ps{l}(\x)$ is subgaussian for every $l \in [L]$.

To formalize the above argument, we first prove the following deterministic lemma. 
\begin{lemma}
  \label{lemma: lb: norm of A(y; M)}
  Let $\y \in \R^{d''}$ and $\M \in \R^{d'' \times d'}$ be arbitrary, and let $\phi_1, \dots, \phi_{d''}: \R \to \R$ be $C^2$ functions satisfying $\norm{\phi_k'}_{L^\infty} \vee \norm{\phi_k''}_{L^\infty} \le C_\phi$ for every $k \in [d'']$.
  Write $\bphi(\y) := (\phi_i(y_i))_{i \in [d'']}$ and define
  \[
    \cA(\y; \M)
    := \begin{bmatrix}
      \bphi( \y + \M_{:, 1} ) - \bphi(\y) &  
      \cdots  & 
      \bphi( \y + \M_{:, d'} ) - \bphi(\y)
    \end{bmatrix} 
    \in \R^{d'' \times d'}.
  \]
  Then, we have 
  \[
    \norm{ \cA(\y; \M) }_{\OP}
    \le C_\phi \norm{ \M }_{\OP} 
      + \frac{C_\phi}{2} \norm{ \M }_{\OP}^2.
  \]
\end{lemma}
\begin{proof}
  Consider the $(i, j)$ entry of $\cA(\y; \M)$. 
  We have 
  \[
    \phi_i(y_i + M_{i, j}) - \phi_i(y_i)
    = \phi_i'(y_i) M_{i, j} + \frac{R_{i, j}}{2} M_{i, j}^2,
  \]
  where $R_{i, j} \in \R$ satisfies $|R_{i,j}| \le \norm{ \phi_i'' }_{L^\infty} \le C_\phi$. 
  Therefore, we can write $\cA(\y; \M)$ as 
  \[
    \cA(\y; \M)
    = \diag( \bphi'(\y) ) \M + \frac{1}{2} \bR \odot \M^{\odot 2},
  \]
  where $\odot$ denotes the Hadamard product and $\M^{\odot 2}$ denotes the entrywise square.  
  Thus, we have 
  \[
    \norm{ \cA(\y; \M) }_{\OP}
    \le \norm{ \diag(\bphi'(\y)) }_{\OP} \norm{ \M }_{\OP} 
      + \frac{1}{2} \norm{ \bR \odot \M^{\odot 2} }_{\OP} 
    \le C_\phi \norm{ \M }_{\OP} 
      + \frac{1}{2} \norm{ \bR \odot \M^{\odot 2} }_{\OP} .
  \]
  For the second term, note that since $\M^{\odot 2}$ is entrywise nonnegative, for any unit vector $\u, \v$, we have 
  \begin{align*}
    \abs{ \u\trans \left( \bR \odot \M^{\odot 2} \right) \v }
    =\abs{  \sum_{i,j} u_i v_j R_{ij} M_{ij}^2 }
    &\le C_\phi \sum_{i,j} |u_i| |v_j| M_{ij}^2 \\
    &= C_\phi |\u|\trans \M^{\odot 2} |\v|
    \le C_\phi \norm{ \M^{\odot 2} }_{\OP}
    \le C_\phi \norm{ \M }_{\OP}^2.
  \end{align*}
  In other words, $\norm{ \bR \odot \M^{\odot 2} }_{\OP} \le C_\phi \norm{ \M }_{\OP}^2$.
\end{proof}

\begin{lemma}
  \label{lemma: lb: norm of Al}
  Suppose that Assumption~\ref{assumption: lower bound assumption} holds.
  For $l \in [L]$, define 
  \[
    \cA_l 
    := \left[ \x\ps{l+1}( \x_{\oplus k} ) - \x\ps{l+1}( \x ) \right]_{k \in [d]} 
      \in \R^{d_{l+1} \times d}, \quad 
    \cA_0 
    := \left[ \x_{\oplus k} - \x \right]_{k \in [d]} \in \R^{d_1 \times d}. 
  \]
  Then, we have $\norm{\cA_0}_{\OP} = 2$ and 
  \[
    \norm{ \cA_l }_{\OP}
    \le C_\phi C_W \norm{ \cA_{l-1} }_{\OP}
      + \frac{C_\phi C_W^2}{2}  \norm{ \cA_{l-1} }_{\OP}^2, \quad 
    \forall l \in [L].
  \]
  As a corollary, we have $\norm{ \cA_l }_{\OP} \le (4 C_\phi C_W^2)^{2^l}$ for all $l \in [L]$.
\end{lemma}
\begin{proof}
  Recall that $\x\ps{l+1}(\x) := \bphi_l(\W\ps{l}\x\ps{l}(\x))$. Write 
  \begin{align*}
    \cA_l 
    &= \left[ 
        \bphi_l\left( \W\ps{l}\x\ps{l}(\x_{\oplus k}) \right) 
        - \bphi_l\left( \W\ps{l}\x\ps{l}(\x) \right) 
      \right]_{k \in [d]} \\
    &= \left[ 
        \bphi_l\left( 
          \W\ps{l}\x\ps{l}(\x) 
          + \W\ps{l} \left( \x\ps{l}(\x_{\oplus k}) - \x\ps{l}(\x) \right)
        \right) 
        - \bphi_l\left( \W\ps{l}\x\ps{l}(\x) \right) 
      \right]_{k \in [d]}. 
  \end{align*}
  Apply Lemma~\ref{lemma: lb: norm of A(y; M)} with $\y = \W\ps{l}\x\ps{l}(\x)$ and $\M = \W\ps{l} \left( \x\ps{l}(\x_{\oplus k}) - \x\ps{l}(\x) \right)_k$, note that
  \[
    \norm{ \M }_{\OP}
    \le \norm{\W\ps{l}}_{\OP} \norm{ \left[ \x\ps{l}(\x_{\oplus k}) - \x\ps{l}(\x) \right]_k }_{\OP}
    \le C_W \norm{ \cA_{l-1} }_{\OP},
  \]
  and we obtain 
  \begin{align*}
    \norm{ \cA_l }_{\OP}
    &\le C_\phi \norm{\W\ps{l}}_{\OP} \norm{ \cA_{l-1} }_{\OP}
      + \frac{\norm{\phi_l''}_{L^\infty}}{2} \norm{\W\ps{l}}_{\OP}^2 \norm{ \cA_{l-1} }_{\OP}^2 \\
    &\le C_\phi C_W \norm{ \cA_{l-1} }_{\OP}
      + \frac{C_\phi C_W^2}{2}  \norm{ \cA_{l-1} }_{\OP}^2. 
  \end{align*}
  In particular, when $\norm{\cA_{l-1}}_{\OP} \ge 1$, the above bound implies 
  \[
    \norm{ \cA_l }_{\OP}
    \le 2 C_\phi C_W^2 \norm{ \cA_{l-1} }_{\OP}^2.
  \]
  When $\norm{\cA_{l-1}}_{\OP} \le 1$, it suffices to replace $\norm{ \cA_{l-1} }_{\OP}$ with $\norm{ \cA_{l-1} }_{\OP} \vee 1$ in the above recurrence.
  Thus, we have $\norm{ \cA_l }_{\OP} \le (4 C_\phi C_W^2)^{2^l}$. 
\end{proof}

\begin{lemma}[Subgaussianity of the intermediate representations]
  \label{lemma: lb: tail bounds for the intermediate representations}
  Suppose that Assumption~\ref{assumption: lower bound assumption} holds.
  Then, for any $l \in [L]$ and $\a \in \S^{d_{l+1}-1}$, we have, for every $t \ge 0$,
  \[
    \P\left[
      \abs{ \inprod{\a}{\x\ps{l+1}(\x)} - \E\inprod{\a}{\x\ps{l+1}(\x)} }
      \ge t
    \right]
    \le 2 \exp\left( - \frac{t^2}{ (4 C_\phi C_W^2)^{2^{l+1}} } \right)
    \le 2 \exp\left( - \frac{t^2}{ 2 \sigma_X^2 } \right), 
  \]
  where $\sigma_X = (4 C_\phi C_W^2)^{2^L}$.
  In other words, $\x\ps{l+1} - \E \x\ps{l+1}$ has a $\sigma_X^2$-subgaussian tail. 
\end{lemma}
\begin{proof}
  By the argument following Lemma~\ref{lemma: bounded difference inequality}, we have 
  \[
    V_{\inprod{\a}{\x\ps{l+1}}}(\x)
    \le \norm{\a}^2 \norm{ \left[ \x\ps{l+1}(\x) - \x\ps{l+1}(\x_{\oplus k}) \right]_{k \in [d]} }_{\OP}^2 
    \le (4 C_\phi C_W^2)^{2^{l+1}}, 
  \]
  where the second inequality comes from Lemma~\ref{lemma: lb: norm of Al}. 
  To complete the proof, it suffices to invoke Lemma~\ref{lemma: bounded difference inequality}. 
\end{proof}

We close this subsection with the following estimate on $\E \x\ps{l+1}(\x)$.

\begin{lemma}
  \label{lemma: lb: | E xl+1 | <=}
  Suppose that Assumption~\ref{assumption: lower bound assumption} holds.
  Then, for any $l \in [L]$, we have 
  \[
    \norm{ \E \x\ps{l+1} }
    \le 2 L \sigma_X (C_\phi C_W)^L \sqrt{d_{\max}}
    =: C_X \sqrt{d_{\max}}, 
  \]
  where $C_X := 2 L \sigma_X (C_\phi C_W)^L$.
\end{lemma}
\begin{proof}
  Recall that $\x\ps{l+1} = \bphi_l\left( \W\ps{l}\x\ps{l} \right)$.
  First, by the Lipschitzness of $\bphi_l$, we have 
  \[
    \norm{ \bphi_l\left( \W\ps{l}\x\ps{l} \right) - \bphi_l\left( \W\ps{l} \E[\x\ps{l}] \right) }
    \le C_\phi C_W \norm{ \x\ps{l} - \E[\x\ps{l}] }.
  \]
  Therefore, 
  \begin{align*}
    \norm{ \E \x\ps{l+1} }
    &\le \norm{ \E \bphi_l\left( \W\ps{l}\x\ps{l} \right) - \bphi_l\left( \W\ps{l} \E[\x\ps{l}] \right) }
      + \norm{ \bphi_l\left( \W\ps{l} \E[\x\ps{l}] \right) } \\
    &\le C_\phi C_W \E \norm{ \x\ps{l} - \E[\x\ps{l}] }
      + \norm{ \bphi_l\left( \W\ps{l} \E[\x\ps{l}] \right) }. 
  \end{align*}
  For the first term, by Lemma~\ref{lemma: lb: tail bounds for the intermediate representations}, we have $C_\phi C_W \E \norm{ \x\ps{l} - \E[\x\ps{l}] } \le C_\phi C_W \sigma_X \sqrt{d_{\max}}$.
  For the second term, by Assumption~\ref{assumption: lower bound assumption}, we have 
  \[
    \norm{ \bphi_l\left( \W\ps{l} \E[\x\ps{l}] \right) } 
    \le \norm{ \bphi_l( 0 ) } + C_\phi C_W \norm{ \E[\x\ps{l}] }
    \le C_\phi \sqrt{d_{\max}} + C_\phi C_W \norm{ \E[\x\ps{l}] }. 
  \]
  As a result, we have 
  \[
    \norm{ \E \x\ps{l+1} }
    \le 2 C_\phi C_W \sigma_X \sqrt{d_{\max}}
      + C_\phi C_W \norm{ \E[\x\ps{l}] }, \quad 
    \norm{ \E \x\ps{1} } = \norm{\E\x} = 0.
  \]
  Solving the above recurrence completes the proof.
\end{proof}

\subsection{Approximating the network}
\label{subsec: lb: approx intermediate repr}

\paragraph{Notation.}
We use $\norm{\cdot}_{L^{p, q}}$ to denote the following norm. 
Let $\h: \R^d \to \R^{d'}$ be an arbitrary (integrable) vector-valued function. 
For any $p, q \in [1, \infty]$, we define 
\[
  \norm{ \h }_{L^{p, q}}
  = \norm{ \norm{ \h(\x) }_p }_{L^q_{\x}}
  = \left( \E \norm{ \h(\x) }_p^q \right)^{1/q}.
\]
That is, $\norm{ \h }_{L^{p, q}}$ is the $L^q$ norm of the $\ell_p$-norm of $\h$. 
When $p = 2$, we will simply write $\norm{\h}_{L^q}$. 

Recall from Section~\ref{sec: lower bound} that our approximating network $\tilde{f}_L$ is constructed by replacing each activation function $\phi_{l, k}$ with its degree-$p_l$ Chebyshev approximation $\tilde{\phi}_{l, k}$ over the interval $I\ps{l}_k := [ \bar{z}\ps{l}_k - R_l, \bar{z}\ps{l}_k + R_l ]$, where $\bar{\z}\ps{l} := \W\ps{l}\E\x\ps{l}$ and $p_l$ is chosen so that the $L^\infty$-difference over $I\ps{l}_k$ is bounded by $\eps_\phi$. 
In this subsection, we show that we can choose $(R_l)_l$ and $\eps_\phi$ so that $f_L \approx \tilde{f}_L$ and $p_l = d^{o(1)}$. 
We will use Lemma~\ref{lemma: lb: tail bounds for the intermediate representations} and Lemma~\ref{lemma: lb: | E xl+1 | <=} extensively in this subsection. 
For convenience, we restate them here:
Under Assumption~\ref{assumption: lower bound assumption}, we have, for every $l \in [L]$, 
\begin{equation}
  \label{eq: lb: properties of xl+1}
  \x\ps{l+1} - \E \x\ps{l+1} \text{ has a $\sigma_X^2$-subgaussian tail and } 
  \norm{ \E\x\ps{l+1} } \le C_X \sqrt{d_{\max}}, 
\end{equation}
where $\sigma_X := (4 C_\phi C_W^2)^{2^L}$ and $C_X := 2 L \sigma_X (C_\phi C_W)^L$.

As discussed in the proof sketch, we will use a layerwise argument. 
Formally, let $\cG_l$ denote the event that in the approximating network, all preactivations at layer $l$ lie in their corresponding approximation intervals $I\ps{l}_k$. 
That is, 
\[
  \cG_l 
  := \braces{
    \x \in \{\pm 1\}^d 
    \,:\,
    \max_{k \in [d_{l+1}]} \abs{ \inprod{\w\ps{l}_k}{\tilde{\x}\ps{l}(\x)} - \bar{z}\ps{l}_k } \le R_l
  }.
\]
For notational simplicity, we also define $\cG_{< l} = \cG_1 \cap \cdots \cap \cG_{l-1}$ and $\cG_{\le l} = \cG_{< l} \cap \cG_l$. 
Then, we define $\cB_l$ to be the event that the first failure happens at layer $l$. 
Namely, 
\[
  \cB_l 
  := \cG_{< l} \cap \cG_l^c 
  = \cG_1 \cap \cdots \cap \cG_{l-1} \cap \cG_l^c. 
\]
Note that the input space can be decomposed as $\{\pm 1\}^d = \cB_1 \cup \cdots \cup \cB_l \cup \cG_{\le l}$ for every $l \in [L]$.
Now, we decompose the error according to this decomposition of input space. 

\begin{lemma}
  \label{lemma: lb: error decomposition}
  Suppose that Assumption~\ref{assumption: lower bound assumption} holds.
  Then, we have   
  \[
    \norm{ f_L - \tilde{f}_L }_{L^2}
    \le L (C_\phi C_W)^L \sqrt{d_{\max}} \eps_\phi 
      + \sum_{l=1}^{L} \P[\cB_l]^{1/4} \norm{ f_L }_{L^4}
      + \sum_{l=1}^{L} \P[\cB_l]^{1/4} \norm{ \tilde{f}_L }_{L^4|\cG_{<l}}
  \]
\end{lemma}
\begin{proof}
  Since $\{ \pm 1 \}^d = \cB_1 \cup \cdots \cup \cB_L \cup \cG_{\le L} $, we have 
  \[
    \norm{ f_L - \tilde{f}_L }_{L^2}
    \le \norm{  \left( f_L - \tilde{f}_L \right) \indi_{\cG_{\le L}} }_{L^2}
      + \sum_{l=1}^{L} \norm{ \indi_{\cB_l} f_L }_{L^2}
      + \sum_{l=1}^{L} \norm{ \indi_{\cB_l} \tilde{f}_L }_{L^2}.
  \]
  Consider the first term. 
  Let $\hat{\x}\ps{l, l'}$ be the version of $\x\ps{l}$ with $\bphi_1, \cdots, \bphi_{l'}$ replaced by $\tilde{\bphi}_1, \dots, \tilde{\bphi}_{l'}$. 
  Note that $\x\ps{l} = \hat{\x}\ps{l, 0}$ and $\tilde{\x}\ps{l} = \hat{\x}\ps{l, l-1}$.
  Hence, for any $l \in [L+1]$, we can write $\tilde{\x}\ps{l} - \x\ps{l}$ as the telescoping sum 
  \[
    \tilde{\x}\ps{l}(\x) - \x\ps{l}(\x)
    = \sum_{l'=1}^{l-1} \left( \hat{\x}\ps{l, l'}(\x) - \hat{\x}\ps{l, l'-1}(\x) \right).
  \]
  In particular, this implies that for any $\x \in \{\pm 1\}^d$, 
  \[
    \abs{ f_L(\x) - \tilde{f}_L(\x) }
    \le \norm{ \x\ps{L+1}(\x) - \tilde{\x}\ps{L+1}(\x) }
    \le \sum_{l=1}^L \norm{ \hat{\x}\ps{L+1, l}(\x) - \hat{\x}\ps{L+1, l-1}(\x) }.
  \]
  Note that in $\hat{\x}\ps{L+1, l}(\x), \hat{\x}\ps{L+1, l-1}(\x)$, the activations $\bphi_{l+1}, \cdots, \bphi_{L}$ are unchanged and all of them are $C_\phi$-Lipschitz functions. 
  Also recall that $\norm{\W\ps{l'}}_{\OP} \le C_W$ for all $l' \in [L]$. 
  Therefore, $\hat{\x}\ps{L+1, l}, \hat{\x}\ps{L+1, l-1}$ are the same $(C_\phi C_W)^L$-Lipschitz function of their corresponding layer-$l$ outputs. 
  As a result, we can further rewrite the above as 
  \begin{equation}
    \label{eq: lb: fL - tildefL, telescoping sum}
    \abs{ f_L(\x) - \tilde{f}_L(\x) }
    \le (C_\phi C_W)^L \sum_{l=1}^L \norm{ 
        \bphi_{l}\left( \W\ps{l} \tilde{\x}\ps{l}(\x) \right) 
        - \tilde\bphi_{l}\left( \W\ps{l} \tilde{\x}\ps{l}(\x) \right)
      }.
  \end{equation}
  Recall that $\tilde{\bphi}_l$ is chosen so that the $L^\infty$-error is at most $\eps_\phi$ in the approximation intervals. 
  Therefore,  
  \[
    \abs{ f_L(\x) - \tilde{f}_L(\x) }
    \le (C_\phi C_W)^L \sum_{l=1}^L \sqrt{d_{l+1}} \eps_\phi
    \le L (C_\phi C_W)^L \sqrt{d_{\max}} \eps_\phi, \quad 
    \forall \x \in \cG_{\le L}.
  \]
  As a result, for the first error term, we have 
  \[
    \norm{  \left( f_L - \tilde{f}_L \right) \indi_{\cG_{\le L}} }_{L^2}
    \le \norm{  \left( f_L - \tilde{f}_L \right) \indi_{\cG_{\le L}} }_{L^\infty}
    \le L (C_\phi C_W)^L \sqrt{d_{\max}} \eps_\phi.
  \]
  Now, consider the second error term $\sum_{l=1}^{L} \norm{ \indi_{\cB_l} f_L }_{L^2}$.
  By the Cauchy-Schwarz inequality, we have 
  \begin{align*}
    \norm{ \indi_{\cB_l} f_L }_{L^2}^2 
    = \E\left[ \indi_{\cB_l}(\x) \indi_{\cG_{<l}}(\x) f_L^2(\x)  \right] 
    &\le \E\left[ \indi_{\cB_l}(\x) \right]^{1/2}
      \E\left[ \indi_{\cG_{<l}}(\x) f_L^4(\x) \right]^{1/2}  \\
    &= \P\left[ \cB_l \right]^{1/2}
      \E\left[ \indi_{\cG_{<l}}(\x) f_L^4(\x) \right]^{1/2}. 
  \end{align*}
  Therefore, 
  \[
    \sum_{l=1}^{L} \norm{ \indi_{\cB_l} f_L }_{L^2}
    \le \sum_{l=1}^{L} \P[\cB_l]^{1/4} \E\left[ \indi_{\cG_{<l}}(\x) f_L^4(\x)\right]^{1/4}
    \le \sum_{l=1}^{L} \P[\cB_l]^{1/4} \norm{ f_L }_{L^4}.  
  \]
  Similarly, we also have 
  \begin{align*}
    \sum_{l=1}^{L} \norm{ \indi_{\cB_l} \tilde{f}_L }_{L^2}
    \le \sum_{l=1}^{L} \P[\cB_l]^{1/4} \E\left[ \indi_{\cG_{<l}}(\x) \tilde{f}_L^4(\x)\right]^{1/4}
    &\le \sum_{l=1}^{L} \P[\cB_l]^{1/4} \E\left[ \tilde{f}_L^4(\x) \mid \cG_{<l} \right]^{1/4} \\
    &= \sum_{l=1}^{L} \P[\cB_l]^{1/4} \norm{ \tilde{f}_L }_{L^4|\cG_{<l}}. 
  \end{align*}
  Combining the above bounds completes the proof.
\end{proof}

By the above lemma, to control the approximation error, it suffices to bound $\P[\cB_l]$ and the (conditional) fourth moments of $f_L$ and $\tilde{f}_L$. 
To this end, we first bound $\P[\cB_l]$ and show that, on the event $\cG_{< l}$, $\tilde{\x}\ps{l} - \E \x\ps{l}$ is subgaussian.

\begin{lemma}
  \label{lemma: lb: bad event probability, cond subg}
  Suppose that Assumption~\ref{assumption: lower bound assumption} holds and 
  \[
    \eps_\phi \le \left( L (C_\phi C_W)^L \sqrt{d_{\max}}  \right)\inv, \quad 
    R_1 \ge \cdots R_L \ge 2 C_W \sigma_X \sqrt{ 2 \log\left( 4 L d_{\max} \right) }.
  \]
  Then, for any $l \in [L]$ and $t \ge 0$, $\v \in \S^{d_{l}-1}$, we have 
  \[
    \P[\cB_l]
    \le 2 d_{\max} \exp\left( - \frac{(R_l / C_W)^2}{2 \cdot 4 \sigma_X^2} \right), \quad 
    \P\left[
      \abs{ \inprod{\v}{ \tilde\x\ps{l} - \E \x\ps{l} } } \ge t 
      \;\big|\;
      \cG_{<l}
    \right]
    \le 2 \exp\left( - \frac{t^2}{2 \cdot 8 \sigma_X^2} \right).
  \]
  In other words, when conditioned on $\cG_{<l}$, $\tilde\x\ps{l} - \E \x\ps{l}$ has an $(8 \sigma_X^2)$-subgaussian tail. 
\end{lemma}
\begin{proof}
  \def\currentprefix{proof: lb: bad event probability, cond subg. ad.as}
  The proof of \eqref{eq: lb: fL - tildefL, telescoping sum}, \textit{mutatis mutandis}, shows that 
  \[
    \norm{ \tilde{\x}\ps{l}(\x) - \x\ps{l}(\x) }
    \le (C_\phi C_W)^L 
      \sum_{l'=1}^{l-1} \norm{ 
        \tilde\bphi_{l'}\left( \W\ps{l'} \tilde{\x}\ps{l'}(\x) \right)
        - \bphi_{l'}\left( \W\ps{l'} \tilde{\x}\ps{l'}(\x) \right)
      }.
  \]
  On event $\cG_{<l}$, for every $l' \in [l-1]$, the entries of every $\W\ps{l'}\tilde{\x}\ps{l'}(\x)$ lie in the corresponding approximation intervals. 
  Therefore, we have $\norm{ \tilde\bphi_{l'}\left( \W\ps{l'} \tilde{\x}\ps{l'}(\x) \right) - \bphi_{l'}\left( \W\ps{l'} \tilde{\x}\ps{l'}(\x) \right) } \le \sqrt{d_{\max}} \eps_\phi$. 
  Thus, 
  \begin{equation}
    \label{eq: lb: tilde xl - xl on G<l}
    \norm{ \tilde{\x}\ps{l}(\x) - \x\ps{l}(\x) }
    \le L (C_\phi C_W)^L \sqrt{d_{\max}} \eps_\phi, \quad 
    \forall \x \in \cG_{<l}.
  \end{equation}
  Fix an arbitrary direction $\v \in \S^{d_l-1}$. 
  The above bound implies that, for any $\x \in \cG_{<l}$,
  \begin{align*}
    \inprod{\v}{ \tilde\x\ps{l} - \E \x\ps{l} }
    = \inprod{\v}{ \x\ps{l} - \E \x\ps{l} } 
      + \inprod{\v}{ \tilde\x\ps{l} - \x\ps{l} } 
    &= \inprod{\v}{ \x\ps{l} - \E \x\ps{l} } 
      \pm L (C_\phi C_W)^L \sqrt{d_{\max}} \eps_\phi \\
    &= \inprod{\v}{ \x\ps{l} - \E \x\ps{l} } 
      \pm 1, 
  \end{align*}
  where the second line comes from the condition $\eps_\phi \le \left( L (C_\phi C_W)^L \sqrt{d_{\max}}  \right)\inv$. 
  Therefore, for any $t \ge 2$, we have 
  \begin{align}
    \P\left[
      \abs{ \inprod{\v}{ \tilde\x\ps{l} - \E \x\ps{l} } } \ge t 
      \;\big|\;
      \cG_{<l}
    \right]
    &\le \P\left[
        \abs{ \inprod{\v}{ \x\ps{l} - \E \x\ps{l} } } \ge t/2 
        \;\big|\;
        \cG_{<l}
      \right] \nonumber \\
    &\le \P\left[
        \abs{ \inprod{\v}{ \x\ps{l} - \E \x\ps{l} } } \ge t/2 
      \right] 
      / \P[ \cG_{<l} ] \nonumber \\
    &\le \frac{2}{ \P[ \cG_{<l} ] }
      \exp\left(
        - \frac{t^2}{2 \cdot 4 \sigma_X^2}
      \right), \locallabel{eq1}
  \end{align}
  where the last line comes from \eqref{eq: lb: properties of xl+1} (or Lemma~\ref{lemma: lb: tail bounds for the intermediate representations}). 
  To turn this into a proper tail bound, we lower-bound the denominator $\P[\cG_{<l}]$. 
  First, by the union bound, we have 
  \[
    1 - \P[\cG_{<l}]
    = \P[ \cB_1 \cup \cdots \cup \cB_{l-1} ]
    \le \sum_{l'=1}^{l-1} \P[\cB_{l'}]. 
  \]
  Furthermore, for each $l'$, we have 
  \[
    \P[\cB_{l'}]
    \le \sum_{k=1}^{d_{l'+1}} 
      \P\left[ 
        \abs{ \inprod{ \w\ps{l'}_k}{ \tilde{\x}\ps{l'}(\x) - \E \x\ps{l'} }  } > R_{l'}
        \;\big|\; \cG_{<l'} 
      \right] 
      \P[ \cG_{< l'} ] .
  \]
  Applying (\localref{eq1}) to each summand yields the unconditioned bound 
  \[
    \P[\cB_{l'}]
    \le 2 \sum_{k=1}^{d_{l'+1}} \exp\left( - \frac{(R_{l'} / C_W)^2}{2 \cdot 4 \sigma_X^2} \right)
    \le 2 d_{\max} \exp\left( - \frac{(R_{l'} / C_W)^2}{2 \cdot 4 \sigma_X^2} \right).
  \]
  This proves the first part of the lemma. 
  Moreover, it implies 
  \[
    1 - \P[\cG_{<l}]
    \le \sum_{l'=1}^{l-1} 2 d_{\max} \exp\left( - \frac{(R_{l'} / C_W)^2}{2 \cdot 4 \sigma_X^2} \right)
    \le 2 L d_{\max} \exp\left( - \frac{(R_{l-1} / C_W)^2}{2 \cdot 4 \sigma_X^2} \right),
  \]
  where the second inequality comes from $R_1 \ge \cdots \ge R_L$. 
  In particular, the right-hand side is at most $1/2$ whenever
  \[
    R_L \ge 2 C_W \sigma_X \sqrt{ 2 \log\left( 4 L d_{\max} \right) }.
  \]
  Combining this fact with (\localref{eq1}) gives
  \[
    \P\left[
      \abs{ \inprod{\v}{ \tilde\x\ps{l} - \E \x\ps{l} } } \ge t 
      \;\big|\;
      \cG_{<l}
    \right]
    \le 4 \exp\left( - \frac{t^2}{2 \cdot 4 \sigma_X^2} \right), 
    \quad \forall t \ge 2.
  \]
  The left-hand side is trivially bounded by $1$. 
  Thus, the above implies that 
  \[
    \P\left[
      \abs{ \inprod{\v}{ \tilde\x\ps{l} - \E \x\ps{l} } } \ge t 
      \;\big|\;
      \cG_{<l}
    \right]
    \le 2 \exp\left( - \frac{t^2}{16 \sigma_X^2} \right), 
    \quad \forall t \ge 0.
  \]
  Indeed, when $t^2 \le 16 (\log 2) \sigma_X^2$, the right-hand side is at least $1$. 
  When $t^2 \ge 16 (\log 2) \sigma_X^2 \ge 1$, we have 
  \[
    \frac{2 \exp\left( - \frac{t^2}{16 \sigma_X^2} \right)}{ 4 \exp\left( - \frac{t^2}{2 \cdot 4 \sigma_X^2} \right) }
    = \frac{1}{2} \exp\left( \frac{t^2}{16 \sigma_X^2} \right) 
    \ge \frac{1}{2} \exp\left( \log 2 \right) 
    = 1. 
  \]
\end{proof}

Now, consider the $L^4$-norm of $f_L$. 
Because $f_L$ uses the exact intermediate representations, we can directly apply Lemma~\ref{lemma: lb: tail bounds for the intermediate representations} to the last layer to control $f_L$. 
\begin{lemma}
  \label{lemma: lb: L4 norm of fL}
  Suppose that Assumption~\ref{assumption: lower bound assumption} holds. 
  Then, we have 
  \[
    \norm{f_L}_{L^4} \le 3 C_X \sqrt{d_{\max}}.
  \]
\end{lemma}
\begin{proof}
  First, we write 
  \[
    \norm{ f_L }_{L^4}^4 
    = \E \inprod{\a}{\x\ps{L+1}}^4
    \le 16 \E \inprod{\a}{\x\ps{L+1} - \E \x\ps{L+1}}^4
      + 16 \E \inprod{\a}{\E \x\ps{L+1}}^4. 
  \]
  By Lemma~\ref{lemma: lb: tail bounds for the intermediate representations} and Lemma~\ref{lemma: lb: | E xl+1 | <=}, $\x\ps{L+1} - \E \x\ps{L+1}$ has a $\sigma_X^2$-subgaussian tail and $\norm{\E \x\ps{L+1}} \le C_X \sqrt{d_{\max}}$.
  As a result, we can further bound the above expression by
  \[
    \norm{ f_L }_{L^4}^4 
    \le 256 \sigma_X^4
      + 16 C_X^4 d_{\max}^2
    \le 20 C_X^4 d_{\max}^2.
  \]
\end{proof}

Finally, we consider the $(L^4|\cG_{<l})$-norm of $\tilde{f}_L$.
To this end, it suffices to control the $(L^r|\cG_{<l'})$-norm of $\tilde{\x}\ps{l}$.
\begin{lemma}
  \label{lemma: lb: tilde x moment bounds}
  Suppose that Assumption~\ref{assumption: lower bound assumption} and the hypotheses of Lemma~\ref{lemma: lb: bad event probability, cond subg} hold, and $R_1 \le \sqrt{d_{\max}}$. 
  For any $l \in [L+1]$, $l' \in [l]$ and $r \ge 1$ with $r p_L \cdots p_1 \le \sqrt{d_{\max}}$, there exists a universal constant $C \ge 1$ such that 
  \[
    \norm{ \tilde{\x}\ps{l}(\x) }_{L^r|\cG_{<l'}}
    \le \left(  C C_\phi C_W C_X d_{\max} \right)^{2 L p_{l-1} \cdots p_{l'} }.
  \]
\end{lemma}
\begin{proof}
  \def\currentprefix{proof: lemma: lb: conditional Lq norm of tilde x. .,adfko}
  By Lemma~\ref{lemma: lb: bad event probability, cond subg}, when conditioned on $\cG_{<l'}$, $\tilde\x\ps{l'} - \E \x\ps{l'}$ has an $(8 \sigma_X^2)$-subgaussian tail. 
  Consider the part of $\tilde{\x}\ps{l}$ after the $l'$-th layer. 
  Recall that $\tilde{\phi}_{l'', k}$ is the degree-$p_{l''}$ Chebyshev approximation to $\phi_{l'',k}$ on $[\bar{z}\ps{l''}_k - R_{l''}, \bar{z}\ps{l''}_k + R_{l''}]$. 
  In other words, we have 
  \[
    \tilde{\phi}_{l'', k}(z)
    = \tilde\psi_{l'', k}\left( \left( z - \bar{z}\ps{l''}_k \right) / R_{l''} \right), 
    \quad\forall z \in \R, \, l'' \in [L],
  \]
  where $\tilde\psi_{l'', k}$ is the degree-$p_{l''}$ Chebyshev approximation to $\psi_{l'', k}(u) = \phi_{l'', k}(\bar{z}\ps{l''}_k + R_{l''} u)$ over $[-1, 1]$. 
  Hence, we can rewrite $\tilde{\x}\ps{l}$ as 
  \[
    \tilde\x\ps{l''+1} 
    = \tilde\bpsi_{l''}\left( R_{l''}\inv \left( \W\ps{l''}\tilde\x\ps{l''} - \bar{\z}\ps{l''} \right) \right), \quad 
    \forall l'' = l', \dots, l-1.
  \]

  We now bound the moments of $\tilde{\x}\ps{l}$ using this recurrence relation. 
  Fix $l'' \in \{ l', \dots, l-1 \}$, $r \ge 1$, and consider the $(L^r|\cG_{<l'})$-norm of $\tilde\x\ps{l''+1}$. 
  Using the fact $\norm{\cdot}_2 \le \norm{\cdot}_1$, we write 
  \begin{equation}
    \locallabel{eq1}
    \norm{\tilde\x\ps{l''+1}}_{L^r|\cG_{<l'}}
    \le \sum_{k=1}^{d_{l''+1}} \norm{\tilde{x}\ps{l''+1}_k}_{L^r|\cG_{<l'}} 
    = \sum_{k=1}^{d_{l''+1}} 
      \E|_{\cG_{<l'}}\left[
        \abs{ \tilde{\psi}_{l'', k} }^r\left(
          R_{l''}\inv \left( \inprod{\w\ps{l''}_k}{\tilde\x\ps{l''}} - \bar{z}\ps{l''}_k \right)
        \right)
      \right]^{1/r}, 
  \end{equation}
  where $\E|_{\cG_{<l'}}[\cdot] := \E[\cdot \mid \cG_{<l'}]$.
  Since $\tilde{\psi}_{l'', k}$ is a polynomial, we can bound its growth using Lemma~\ref{lemma: poly <= Tn}, which only requires a bound on $\tilde{\psi}_{l'', k}|_{[-1, 1]}$. 
  Since $\tilde{\psi}_{l'', k}$ is chosen to approximate $\psi_{l'', k}$ to $\eps_\phi$-accuracy on $[-1, 1]$, it suffices to bound $\psi_{l'', k}$. 
  For any $u \in [-1, 1]$, by Assumption~\ref{assumption: lower bound assumption}\ref{itm: lb: sigma derivatives bounds}, we have 
  \begin{align*}
    \abs{ \psi_{l'', k}(u) }
    = \abs{ \phi_{l'', k}( \bar{z}\ps{l''}_k + R_{l''} u ) }
    &\le |\phi_{l'', k}(0)|
      + C_\phi\left( \abs{\bar{z}\ps{l''}_k} + R_{l''}  \right) \\
    &\le C_\phi\left( 1 + \abs{\bar{z}\ps{l''}_k} + R_{l''}  \right) \\
    &\le C_\phi\left( 1 + C_W C_X \sqrt{d_{\max}} + R_{l''}  \right), 
  \end{align*}
  where the last line comes from Lemma~\ref{lemma: lb: | E xl+1 | <=} and 
  \begin{equation}
    \label{eq: lb: |bar z| <=}
    \abs{\bar{z}\ps{l''}_k}
    = \abs{ \w\ps{l''}_k \cdot \E \x\ps{l''} }
    \le C_W \norm{ \E \x\ps{l''} }
    \le C_W C_X \sqrt{d_{\max}}.
  \end{equation}
  Also, recall that $R_{l''} \le R_1$. 
  Therefore, we have 
  \[
    \norm{ \tilde\psi_{l'', k} }_{L^\infty([-1, 1])}
    \le 1 + \norm{ \psi_{l'', k} }_{L^\infty([-1, 1])}
    \le 2 C_\phi\left( 1 + C_W C_X \sqrt{d_{\max}} + R_1 \right)
    =: C_{\Tmp}.
  \]
  Since $\tilde\psi_{l'', k}$ is a degree-$p_{l''}$ polynomial, by Lemma~\ref{lemma: poly <= Tn}, this implies
  \[
    \abs{ \tilde\psi_{l'', k}(u) }
    \le C_{\Tmp} \left( 1 + 2 |u| \right)^{p_{l''}}, \quad 
    \forall u \in \R.
  \]
  Combining this with (\localref{eq1}), we obtain 
  \[
    \norm{\tilde\x\ps{l''+1}}_{L^r|\cG_{<l'}}
    \le C_{\Tmp}  \sum_{k=1}^{d_{l''+1}} 
      \E|_{\cG_{<l'}} \left[
        \left( 
          1 
          + 2 \abs{
            R_{l''}\inv \left( \inprod{\w\ps{l''}_k}{\tilde\x\ps{l''}} - \bar{z}\ps{l''}_k \right)
          }
        \right)^{r p_{l''}}
      \right]^{1/r}. 
  \]
  For the term inside the expectation, by repeatedly using the inequality $(a + b)^p \le 2^p (|a|^p + |b|^p)$, we obtain 
  \begin{align*}
    & \left( 
      1 
      + 2 \abs{
        R_{l''}\inv \left( \inprod{\w\ps{l''}_k}{\tilde\x\ps{l''}} - \bar{z}\ps{l''}_k \right)
      } 
    \right)^{r p_{l''}} \\
    \le\;& 2^{r p_{l''}}
        \left( 
          1 
          + 2^{r p_{l''}} 
          \abs{
            R_{l''}\inv \left( \inprod{\w\ps{l''}_k}{\tilde\x\ps{l''}} - \bar{z}\ps{l''}_k \right)
          }^{r p_{l''}}
        \right) \\
    \le\;& 2^{r p_{l''}}
        \left( 
          1 
          + (4/R_{l''})^{r p_{l''}} 
          \abs{ \inprod{\w\ps{l''}_k}{\tilde\x\ps{l''}}  }^{r p_{l''}}
          + (4/R_{l''})^{r p_{l''}} 
          \abs{ \bar{z}\ps{l''}_k  }^{r p_{l''}}
        \right) \\
    \le\;& 2^{r p_{l''}}
        \left( 
          1 
          + \abs{ \inprod{\w\ps{l''}_k}{\tilde\x\ps{l''}}  }^{r p_{l''}}
          + \left( C_W C_X \sqrt{d_{\max}} \right)^{r p_{l''}}
        \right), 
  \end{align*}
  where the last line comes from $R_{l''} \gg 1$ and \eqref{eq: lb: |bar z| <=}. 
  As a result, we have 
  \begin{align}
    \norm{\tilde\x\ps{l''+1}}_{L^r|\cG_{<l'}}
    &\le C_{\Tmp}  2^{p_{l''}}
      \sum_{k=1}^{d_{l''+1}}
      \left(
          1 
          + \E|_{\cG_{<l'}} \left[ \abs{ \inprod{\w\ps{l''}_k}{\tilde\x\ps{l''}}  }^{r p_{l''}} \right]
          + \left( C_W C_X \sqrt{d_{\max}} \right)^{r p_{l''}}
      \right)^{1/r} \nonumber \\
    &\le C_{\Tmp}  2^{p_{l''}}
      \sum_{k=1}^{d_{l''+1}}
      \left(
          1 
          + C_W^{r p_{l''}}
          \E|_{\cG_{<l'}} \norm{\tilde\x\ps{l''}}^{r p_{l''}} 
          + \left( C_W C_X \sqrt{d_{\max}} \right)^{r p_{l''}}
      \right)^{1/r} \nonumber \\
    &\le C_{\Tmp}  
      d_{\max}
      \left(
          \left( 4 C_W C_X \sqrt{d_{\max}} \right)^{p_{l''}}
          + \left( 2 C_W \norm{\tilde\x\ps{l''}}_{L^{r p_{l''}}|\cG_{<l'} } \right)^{p_{l''}}
      \right).  \locallabel{eq2}
  \end{align}
  We now inductively construct an upper bound $B(l''; r')$ for each $\norm{\tilde\x\ps{l''}}_{L^{r'}\mid\cG_{<l'}}$, where $l'' \in \{l', \dots, l-1\}$, $r' \in \bbN_+$.
  Define these bounds recursively by
  \begin{align*}
    B(l'; r')
    &:= \norm{ \tilde\x\ps{l'} }_{L^{r'}\mid\cG_{<l'}} \vee 2 C_X \sqrt{d_{\max}},  
    && \forall r' \in \bbN_+, \\
    B(l''+1; r')
    &:= C_{\Tmp} d_{\max} (4 C_W)^{p_{l''}} B(l''; r' p_{l''})^{p_{l''}}, 
    && \forall r' \in \bbN_+, \, l'' \in \{ l', \dots, l-1 \}.
  \end{align*}
  Here, the boundary condition is chosen so that $B(l'; r')\ge \norm{ \tilde\x\ps{l'} }_{L^{r'}\mid\cG_{<l'}}$ and we can merge the $4 C_W C_X \sqrt{d_{\max}}$ part into the  $2 C_W B(l''; r p_{l''})^{p_{l''}}$ part in (\localref{eq2}).
  By (\localref{eq2}), we have $B(l''; r') \ge \norm{ \tilde\x\ps{l''} }_{L^{r'} \mid\cG_{<l'}}$.
  Unrolling this recurrence, we obtain 
  \begin{align*}
    B(l; r)
    &= C_{\Tmp} d_{\max} (4C_W)^{p_{l-1}} B(l-1; r p_{l-1})^{p_{l-1}} \\
    &= \left( C_{\Tmp} d_{\max} \right)^{1 + p_{l-1}}
      (4C_W)^{p_{l-1} + p_{l-1}p_{l-2}} 
      B(l-2; r p_{l-1} p_{l-2})^{p_{l-1} p_{l-2}}  \\
    &= \left(C_{\Tmp} d_{\max} \right)^{1 + p_{l-1} + \cdots + p_{l-1} \cdots p_{l' + 1} }
      (4C_W)^{p_{l-1} + \cdots + p_{l-1} \cdots p_{l'}} 
      B(l'; r p_{l-1} \cdots p_{l'})^{p_{l-1} \cdots p_{l'}} \\
    &\le \left( 4 C_W C_{\Tmp} d_{\max} \right)^{L p_{l-1} \cdots p_{l'} }
      B(l'; r p_{l-1} \cdots p_{l'})^{p_{l-1} \cdots p_{l'}}. 
  \end{align*}
  To upper bound $B(l'; r p_{l-1} \cdots p_{l'})$, recall that $\norm{\E \x\ps{l'}} \le C_X \sqrt{d_{\max}}$ by Lemma~\ref{lemma: lb: | E xl+1 | <=}, and, by Lemma~\ref{lemma: lb: bad event probability, cond subg}, when conditioned on $\cG_{<l'}$, $\tilde{\x}\ps{l'} - \E \x\ps{l'}$ has an $(8 \sigma_X^2)$-subgaussian tail. 
  Hence, 
  \begin{align*}
    \norm{ \tilde\x\ps{l'} }_{L^{r p_{l-1} \cdots p_{l'}}|\cG_{<l'}}
    &\le \norm{ \tilde{\x}\ps{l'} - \E \x\ps{l'} }_{L^{r p_{l-1} \cdots p_{l'}}|\cG_{<l'}}
      + \norm{\E \x\ps{l'}} \\
    &\lesssim \sigma_X \sqrt{d_{\max}} + \sigma_X \sqrt{ r p_{l-1} \cdots p_{l'} } + C_X \sqrt{d_{\max}} \\
    &\lesssim C_X \sqrt{d_{\max}}, 
  \end{align*}
  where the last line comes from the condition $r p_L \cdots p_1 \le \sqrt{d_{\max}}$. 
  This implies that 
  \[
    B(l'; r p_{l-1} \cdots p_{l'})
    = \norm{ \tilde\x\ps{l'} }_{L^{r p_{l-1} \cdots p_{l'}}\mid\cG_{<l'}} \vee 2 C_X \sqrt{d_{\max}}
    \le C C_X \sqrt{d_{\max}}, 
  \]
  for some universal constant $C \ge 1$. 
  Thus, we have 
  \[
    B(l; r)
    \le \left( 4 C_W C_{\Tmp} d_{\max} \right)^{L p_{l-1} \cdots p_{l'} }
      \left( C C_X \sqrt{d_{\max}} \right)^{p_{l-1} \cdots p_{l'}}  
    \le \left(  C C_\phi C_W^2 C_X^2  d_{\max}^2  \right)^{L p_{l-1} \cdots p_{l'} }.
  \]
\end{proof}

\begin{corollary}
  \label{cor: lb: L4|G<l norm of tilde f}
  Under the conditions of Lemma~\ref{lemma: lb: tilde x moment bounds}, there is a universal constant $C \ge 1$ such that 
  \[  
    \norm{ \tilde{f}_L }_{L^4|\cG_{<l}}
    \le \left(  C C_\phi C_W C_X d_{\max} \right)^{2 L p_L \cdots p_{l} }, \quad 
    \forall l \in [L+1].
  \]
\end{corollary}
\begin{proof}
  By Lemma~\ref{lemma: lb: tilde x moment bounds}, we have 
  \[
    \norm{ \tilde{\x}\ps{L+1}(\x) }_{L^r|\cG_{<l}}
    \le \left(  C C_\phi C_W C_X d_{\max} \right)^{2 L p_L \cdots p_{l} }, \quad 
    \forall l \in [L+1].
  \]
  Therefore, 
  \[
    \norm{ \tilde{f}_L }_{L^4|\cG_{<l}}
    = \norm{ \inprod{\a}{\tilde{\x}\ps{L+1}} }_{L^4|\cG_{<l}}
    \le \norm{ \tilde{\x}\ps{L+1} }_{L^4|\cG_{<l}}
    \le \left(  C C_\phi C_W C_X d_{\max} \right)^{2 L p_L \cdots p_{l} }.
  \]

\end{proof}

We now combine the above bounds to prove Lemma~\ref{lemma: lb: approx f with poly}.
For convenience, we restate it below. 
\LBApproxFwithPoly*
\begin{proof}
  \def\currentprefix{proof: lb: approx f with poly. ;lkansdf}
  First, recall from Lemma~\ref{lemma: lb: error decomposition} that 
  \[
    \norm{ f_L - \tilde{f}_L }_{L^2}
    \le L (C_\phi C_W)^L \sqrt{d_{\max}} \eps_\phi 
      + \sum_{l=1}^{L} \P[\cB_l]^{1/4} \norm{ f_L }_{L^4}
      + \sum_{l=1}^{L} \P[\cB_l]^{1/4} \norm{ \tilde{f}_L }_{L^4|\cG_{<l}}.
  \]
  By Lemma~\ref{lemma: lb: bad event probability, cond subg}, Lemma~\ref{lemma: lb: L4 norm of fL}, and Corollary~\ref{cor: lb: L4|G<l norm of tilde f}, we can further bound the right-hand side by
  \begin{align*}
    \norm{ f_L - \tilde{f}_L }_{L^2}
    &\le L (C_\phi C_W)^L \sqrt{d_{\max}} \eps_\phi 
      + \sum_{l=1}^{L} 2 d_{\max} \exp\left( - \frac{(R_l / C_W)^2}{8 \cdot 4 \sigma_X^2} \right) 
      3 C_X \sqrt{d_{\max}} \\
      &\qquad
      + \sum_{l=1}^{L} 2 d_{\max} \exp\left( - \frac{(R_l / C_W)^2}{8 \cdot 4 \sigma_X^2} \right)
        \left(  C C_\phi C_W C_X d_{\max} \right)^{2 L p_L \cdots p_{l} }  \\
    &=: \Tmp_1 + \Tmp_2 + \Tmp_3.
  \end{align*}
  For the first term to be smaller than $\eps_* / 3$, it suffices to choose 
  \begin{equation}
    \locallabel{eq1}
     \eps_\phi 
    := \frac{\eps_*}{3 L (C_\phi C_W)^L \sqrt{d_{\max}}}.
  \end{equation}
  For the second term, since $R_L \le \cdots \le R_1$, we have 
  \[
    \Tmp_2 
    \le L \cdot 2 d_{\max} \exp\left( - \frac{(R_L / C_W)^2}{8 \cdot 4 \sigma_X^2} \right) 
      3 C_X \sqrt{d_{\max}}. 
  \]
  Thus, for $\Tmp_2$ to be bounded by $\eps_* / 3$, it suffices to require
  \begin{equation}
    \locallabel{eq2}
    R_L 
    \ge 20 C_W \sigma_X \sqrt{ \log\left( L C_X d_{\max} / \eps_*  \right) }.
  \end{equation}
  Finally, consider the third error term. 
  For each $l \in [L]$, by Assumption~\ref{assumption: lower bound assumption}\ref{itm: lb: analytic, strip growth} and Lemma~\ref{lemma: chebyshev approx: strip}, to achieve $\eps_\phi$-accuracy over $[\bar{z}\ps{l}_k - R_l, \bar{z}\ps{l}_k + R_l]$, it suffices to choose 
  \[
    \frac{4 R_l}{c_\phi} 
    \log\left( 
      \frac{4 R_l }{c_\phi \eps_\phi} 
      C_\phi \left( 1 + \abs{ \bar{z}\ps{l}_k } + R_{l} + c_\phi/2 \right)^{m_\phi} 
    \right)
    \le C_1 R_{l} \log( d_{\max}/\eps_\phi )
    =: p_l, 
  \]
  where $C_1 \ge 1$ is a large universal constant and $\eps_\phi$ is given by (\localref{eq1}). 
  We then bound $\Tmp_3$ by
  \begin{align*}
    \Tmp_3 
    &\le L \max_{l \in [L]} 2 d_{\max} \exp\left( - \frac{(R_l / C_W)^2}{8 \cdot 4 \sigma_X^2} \right)
        \left(  C C_\phi C_W C_X d_{\max} \right)^{2 L p_L \cdots p_{l} }  \\
    &\le 2 L d_{\max}
        \max_{l \in [L]}    
        \exp\left( - \frac{(R_l / C_W)^2}{8 \cdot 4 \sigma_X^2} \right)
        \left(  C C_\phi C_W C_X d_{\max} \right)^{
          2 L (C_1 \log(d_{\max} / \eps_\phi))^L
          R_L \cdots R_l
        }  \\
    &=: 2 L d_{\max} \max_{l \in [L]} \Tmp_{3, l}.
  \end{align*}
  Now, we choose $(R_l)_l$ so that $\Tmp_{3, l} \le \eps_* / (6 L d_{\max})$ for all $l \in [L]$.
  For notational simplicity, define 
  \begin{equation}
    \locallabel{eq: Lambda}
    \Lambda := 2 L (C_1 \log(d_{\max} / \eps_\phi))^L \log(C C_\phi C_W C_X d_{\max}).
  \end{equation}
  Then, we can rewrite $\Tmp_{3, l}$ as 
  \[
    \Tmp_{3, l}
    = \exp\left( - \frac{(R_l / C_W)^2}{8 \cdot 4 \sigma_X^2} + \Lambda R_L \cdots R_l \right), \quad 
    \forall l \in [L].
  \]
  Suppose that $R_L, \cdots, R_{l+1}$ have been chosen. 
  Then, for $\Tmp_{3, l}$ to be bounded by $\eps_* / (6 L d_{\max})$, it suffices to require
  \begin{multline*}
    \frac{(R_l / C_W)^2}{8 \cdot 4 \sigma_X^2} \ge 2 \Lambda R_L \cdots R_l, \quad 
    \exp\left( - \frac{1}{2} \frac{(R_l / C_W)^2}{8 \cdot 4 \sigma_X^2} \right)
    \le \frac{\eps_*}{6 L d_{\max}}, \\
    \Leftarrow\quad 
    R_l \ge 64 C_W^2 \sigma_X^2 \Lambda R_L \cdots R_{l+1}, \quad 
    R_l \ge 8 C_W \sigma_X \sqrt{ \log\left( 6 L d_{\max} / \eps_* \right) }. 
  \end{multline*}
  Since we require $R_L \le \cdots \le R_1$, we can replace $R_l$ with $R_L$ in the second condition. 
  Combining this with (\localref{eq2}) yields the following backward recurrence:
  \begin{align*}
    R_L 
    &\ge 20 C_W \sigma_X \sqrt{ \log\left( L C_X d / \eps_*  \right) }
      \vee 64 C_W^2 \sigma_X^2 \Lambda 
      \vee 8 C_W \sigma_X \sqrt{ \log\left( 6 L d_{\max} / \eps_* \right) }, \\
    R_l 
    &\ge 64 C_W^2 \sigma_X^2 \Lambda R_L \cdots R_{l+1}.
  \end{align*}
  Recall the definition of $\Lambda$ from (\localref{eq: Lambda}). 
  To satisfy the condition on $R_L$, it suffices to choose
  \begin{equation}
    \locallabel{eq: RL}
    R_L 
    = C_2 C_W^2 \sigma_X^2 \log^{L+1}(d_{\max} / \eps_\phi),
  \end{equation}
  for some large universal constant $C_2 \ge 1$. 
  Similarly, we can simplify the recurrence relation as 
  \[
    R_l 
    \ge C_3 C_W^2 \sigma_X^2 \log^{L+1}(d_{\max} / \eps_\phi) 
      R_L \cdots R_{l+1}, 
  \]
  where $C_3 \ge 1$ is a large universal constant. 
  Solving this recurrence relation gives 
  \begin{equation}
    \locallabel{eq: Rl}
    R_l 
    = \left( 
      C_3 C_W^2 \sigma_X^2 \log^{L+1}(d_{\max} / \eps_\phi)
      \cdot R_L 
    \right)^{2^{L-l-1}}, 
    \quad \forall l \in [L-1].
  \end{equation}
  In particular, this implies that for every $l \in [L]$, 
  \begin{align*}
    p_l 
    = C_1 R_l \log( d_{\max}/\eps_\phi )
    &\le C_1 
      \left( 
        C_3 C_W^2 \sigma_X^2 \log^{L+1}(d_{\max} / \eps_\phi)
        \cdot R_L 
      \right)^{2^L}
      \log( d_{\max}/\eps_\phi ) \\
    &\le C_4 
      \left( 
        C_W^4 \sigma_X^4 \log^{2L+2}(d_{\max} / \eps_\phi)
      \right)^{2^L+1}.
  \end{align*}
  Thus, when $\eps_\phi \ge d^{-O(1)}$ and $d_{\max} \le d^{O(1)}$, the total degree of $\tilde{f}$ is bounded by 
  \[
    \prod_{l=1}^{L} p_l 
    \le C_5 
      \left( 
        C_W^4 \sigma_X^4 
        (\log d)^{2L+2}
      \right)^{L(2^L+1)}
    \le C_6
      \left( 
        (C_\phi C_W^2)^{2^{L+3}}
        (\log d)^{2L+2}
      \right)^{L(2^L+1)}, 
  \]
  where $C_4, C_5, C_6 \ge 1$ are universal constants. 
  Recall from Assumption~\ref{assumption: lower bound assumption} that $C_\phi, C_W \le d^{c_0}$ with $c_0 < 2^{-2L-7}/L$.
  Hence, the above $(p_l)_l$ meet the condition $p_L \cdots p_1 \ll d^{1/2}$ required by Lemma~\ref{lemma: lb: tilde x moment bounds}.
\end{proof}

\subsection{Proof of the main lower bound theorem}
\label{subsec: lb: proof of the main theorem}

\begin{proof}[Proof of Theorem~\ref{thm: lower bound against constant-depth networks}]
  By Lemma~\ref{lemma: lb: approx f with poly}, we can approximate $f_L$ to accuracy $1/d$ by a polynomial $\tilde{f}_L$ whose degree is much smaller than $d^{1/2}$.
  Thus, by the orthogonality of Boolean monomials, we have 
  \[
    \norm{ f_* - f_L }_{L^2}
    \ge \norm{ f_* - \tilde{f}_L }_{L^2}
      - \norm{ f_L - \tilde{f}_L }_{L^2}
    \ge \norm{ \Proj_{> \deg(\tilde{f}_L)} f_*  }_{L^2}
      - (1/d)
    \ge 1/4.
  \]
\end{proof}

\section{Upper bound proofs}
\label{appendix: upper bound proofs}

As we have discussed in Section~\ref{sec: upper bound}, to prove Theorem~\ref{thm: upper bound}, it suffices to show by induction that, for each layer $l \in [L]$, $\x\ps{l}$ contains the features present in $\cK_l \setminus \{\varnothing\}$. 
Formally, we maintain the induction hypothesis Assumption~\ref{assumption: training induction hypothesis}, which we restate below for convenience. 
\UpperBoundIH*

\begin{fact}[Base case]
  \label{fact: ub: base case}
  For $l = 1$, Assumption~\ref{assumption: training induction hypothesis} holds with $I_1 = [d]$, and $\alpha\ps{1}_k = 1$, $S\ps{1}_k = \{k\}$ for all $k \in [d]$.
\end{fact}

Consider the intermediate feature vector $\z\ps{l} := \bPhi( \x\ps{l} )$, where $\bPhi$ is the degree-$D$ polynomial feature mapping defined in \eqref{eq: poly feature mapping}. 
We index it using $H \in \binom{[d + m]}{\le D} \simeq [M]$. 
We also fix an arbitrary ordering of $\binom{[d+m]}{\le D}$ in the inner loop of Algorithm~\ref{alg: training alg} (line~\ref{line: coordinate descent for loop}), and write $H < H'$ if $H$ appears before $H'$ in this order. 
Also, recall from \eqref{eq: Hl} that $\cH_l$ denotes the collection of nontrivial entries of $\z\ps{l}$.

To implement the proof sketch in Section~\ref{sec: upper bound}, we start with a lemma that isolates the relevant terms in the loss \eqref{eq: training loss}.

\begin{lemma}[Loss decomposition]
  \label{lemma: ub: loss decomposition}
  Suppose that Assumption~\ref{assumption: training induction hypothesis} holds at layer~$l$. 
  Consider the $l$th iteration of the outer for loop of Algorithm~\ref{alg: training alg}. 
  Fix a nonzero coordinate $H \in [M]$ of $\z\ps{l}$. 
  Let $\cC_H := \{ H' \in \cH_l \,:\, S\ps{l}_H = S\ps{l}_{H'} \}\subset \cH_l$ denote the collection of the coordinates of $\z\ps{l}$ that represent the same feature as $H$. 
  Then, we can decompose the loss \eqref{eq: training loss} as 
  \[
    \Loss\ps{l}\big(\w\ps{l}\big)
    = \frac{1}{2} \left(
        \hat{f}_*\big( S\ps{l}_H \big) 
        - \sum_{H' \in \cC_H} \alpha\ps{l}_{H'} w\ps{l}_{H'} 
      \right)^2 
      - \alpha\ps{l}_{H} w\ps{l}_{H} \cE\ps{l}_H
      + \left( \text{$w\ps{l}_H$-independent terms} \right), 
  \]
  where 
  \[
    \cE\ps{l}_H
    := \sum_{S \in \cS_* \setminus \{ S\ps{l}_H \}}  \hat{f}_*(S)
        \hat{\E}\ps{l}\left[ \chi_{S\ps{l}_H}(\x) \chi_S(\x) \right]
        - \sum_{G \in \cH_l \setminus \cC_H} 
      \alpha\ps{l}_G   w\ps{l}_G 
      \hat{\E}\ps{l}\left[ \chi_{S\ps{l}_H}(\x) \chi_{S\ps{l}_G}(\x) \right].
  \]
\end{lemma}
\begin{remark}
  Note that the first term in the decomposition depends only on $(w\ps{l}_{H'})_{H' \in \cC_H}$ and is a simple quadratic function in $w\ps{l}_H$.
  Moreover, it does not depend on the sampled dataset. 
  The only (relevant) source of randomness comes from the $\cE\ps{l}_H$ part of the second term.
  By the orthogonality of  Boolean monomials, the $\hat{\E}\ps{l}$ parts in $\cE\ps{l}_H$ have expectation zero. 
  We will show in the next lemma that when the number of samples is reasonably large, $\cE\ps{l}_H$  is uniformly small with high probability for all $\w\ps{l}$ with $\norm{\w\ps{l}}_\infty$ bounded.
\end{remark}
\begin{proof}
  First, using Assumption~\ref{assumption: training induction hypothesis} and \eqref{eq: zlH}, we can rewrite $f\ps{l}(\x) = \inprod{\w\ps{l}}{\z\ps{l}}$ as 
  \[
    f\ps{l}(\x)
    = \sum_{G \in \cH_l} \alpha\ps{l}_G w\ps{l}_G \chi_{S\ps{l}_G}(\x)
    = \chi_{S\ps{l}_H}(\x) \sum_{H' \in \cC_H} \alpha\ps{l}_{H'} w\ps{l}_{H'} 
      + \sum_{G \in \cH_l \setminus \cC_H} \alpha\ps{l}_G w\ps{l}_G \chi_{S\ps{l}_G}(\x).
  \]
  Similarly, we can rewrite the target function as 
  \[
    f_* 
    = \hat{f}_*\big( S\ps{l}_H \big) \chi_{S\ps{l}_H} 
      + \sum_{S \in \cS_* \setminus \{ S\ps{l}_H \}} \hat{f}_*(S) \chi_S. 
  \]
  By the definition of $\cC_H$, $S\ps{l}_H$ differs from every $S\ps{l}_G$ and $S$ appearing in the corresponding sums, and $H \ne G$.
  Then, we decompose the loss as 
  \begin{align*}
    \Loss\ps{l}\big(\w\ps{l}\big)
    &= \frac{1}{2} \hat{\E}\ps{l} \left(
        f_*(\x)
        - \chi_{S\ps{l}_H}(\x) \sum_{H' \in \cC_H} \alpha\ps{l}_{H'} w\ps{l}_{H'} 
        - \sum_{G \in \cH_l \setminus \cC_H} \alpha\ps{l}_G w\ps{l}_G \chi_{S\ps{l}_G}(\x)
      \right)^2  \\
    &= \frac{1}{2} \hat{\E}\ps{l} \left(
        f_*(\x) - \chi_{S\ps{l}_H}(\x) \sum_{H' \in \cC_H} \alpha\ps{l}_{H'} w\ps{l}_{H'} 
      \right)^2 \\
      &\qquad
      + \alpha\ps{l}_{H} w\ps{l}_{H}  
      \sum_{G \in \cH_l \setminus \cC_H} 
      \alpha\ps{l}_G   w\ps{l}_G 
      \hat{\E}\ps{l}\left[ \chi_{S\ps{l}_H}(\x) \chi_{S\ps{l}_G}(\x) \right]
      + \left( \text{$w\ps{l}_H$-independent terms} \right). 
  \end{align*}
  For the first term, we can further decompose it as 
  \begin{multline*}
    \frac{1}{2} \hat{\E}\ps{l} \left(
        f_*(\x) - \chi_{S\ps{l}_H}(\x) \sum_{H' \in \cC_H} \alpha\ps{l}_{H'} w\ps{l}_{H'} 
      \right)^2 \\
    = \frac{1}{2} \left(
        \hat{f}_*\big( S\ps{l}_H \big) 
        - \sum_{H' \in \cC_H} \alpha\ps{l}_{H'} w\ps{l}_{H'} 
      \right)^2 
      - \alpha\ps{l}_{H} w\ps{l}_{H}  \sum_{S \in \cS_* \setminus \{ S\ps{l}_H \}}  \hat{f}_*(S)
        \hat{\E}\ps{l}\left[ \chi_{S\ps{l}_H}(\x) \chi_S(\x) \right]  \\
    + \left( \text{$w\ps{l}_H$-independent terms} \right).
  \end{multline*}
  Combining the preceding two displays completes the proof.
\end{proof}

We now control the error term $\cE\ps{l}_H$ in the above decomposition. 
\begin{lemma}[Error term concentration]
  \label{lemma: ub: error term concatenation}
  Let $\eps_{\cE}, \delta_{\P} \in (0, 1)$ be the target accuracy and failure probability, respectively. 
  Let $B_w \ge 1$ be a parameter controlling the largest possible $\norm{\w\ps{l}}_\infty$. 
  Suppose that 
  \[
    N 
    \gtrsim \frac{B_w^2 M^2}{\eps_{\cE}^2} \log\left( \frac{M}{\delta_{\P}} \right) .
  \]
  Then, with probability at least $1 - \delta_{\P}$, $\big| \cE\ps{l}_H \big| \le \eps_{\cE}$ holds uniformly for all $\w\ps{l}$ with $\norm{\w\ps{l}}_\infty \le B_w$ and $H \in \cH_l$.
\end{lemma}
\begin{proof}
  Recall from Lemma~\ref{lemma: ub: loss decomposition} that 
  \[
    \cE\ps{l}_H
    := \sum_{S \in \cS_* \setminus \{ S\ps{l}_H \}}  \hat{f}_*(S)
        \hat{\E}\ps{l}\left[ \chi_{S\ps{l}_H}(\x) \chi_S(\x) \right]
        - \sum_{G \in \cH_l \setminus \cC_H} 
      \alpha\ps{l}_G   w\ps{l}_G 
      \hat{\E}\ps{l}\left[ \chi_{S\ps{l}_H}(\x) \chi_{S\ps{l}_G}(\x) \right].
  \]
  By definition, $S\ps{l}_H \ne S$ for any $S \in \cS_* \setminus \{ S\ps{l}_H \}$, and $S\ps{l}_H \ne S\ps{l}_G$ for any $G \in \cH_l \setminus \cC_H$. 
  
  Now, consider two arbitrary distinct subsets $S,S' \subset [d]$ and the random variable $\chi_S(\x) \chi_{S'}(\x)$. 
  Note that $\chi_S(\x)\chi_{S'}(\x) \sim \Unif(\{\pm 1\})$. 
  Therefore, by Hoeffding's inequality,
  \(
    \abs{\hat{\E}_N\left[ \chi_S(\x)\chi_{S'}(\x) \right] } \le \eps, 
  \)
  with probability at least $1 - \delta_{\P}$, if $N \ge 2 \log(2 / \delta_{\P}) / \eps^2$. 

  There are at most $M m$ choices of $(H,S) \in [M] \times \cS_*$ and $M^2$ choices of $(H,G) \in [M]\times [M]$. Thus, the union bound implies that
  \[  
    \abs{ \hat{\E}\ps{l}\left[ \chi_{S\ps{l}_H}(\x) \chi_S(\x) \right] } 
    \vee \abs{ \hat{\E}\ps{l}\left[ \chi_{S\ps{l}_H}(\x) \chi_{S\ps{l}_G}(\x) \right] }
    \le \eps, \quad 
    \forall 
    H \in \cH_l,\, 
    S \in \cS_* \setminus \{S\ps{l}_H\}, \,
    G \in \cH_l \setminus \cC_H,
  \]
  with probability at least $1 - 2M^2 \delta_{\P}$, as long as $N \ge 2 \log(2 / \delta_{\P}) / \eps^2$.

  Condition on this event.
  For any $\w\ps{l}$ with $\norm{\w\ps{l}}_\infty \le B_w$, by the Cauchy-Schwarz inequality, Parseval's theorem and Assumption~\ref{assumption: training induction hypothesis}\ref{itm: training IH: features}, we have 
  \begin{align*}
    \abs{ \cE\ps{l}_H }
    &\le \sum_{S \in \cS_* \setminus \{ S\ps{l}_H \}}  \abs{ \hat{f}_*(S) }
        \abs{ \hat{\E}\ps{l}\left[ \chi_{S\ps{l}_H}(\x) \chi_S(\x) \right] }
      + \sum_{G \in \cH_l \setminus \cC_H} 
      \abs{ \alpha\ps{l}_G   w\ps{l}_G }
      \abs{ \hat{\E}\ps{l}\left[ \chi_{S\ps{l}_H}(\x) \chi_{S\ps{l}_G}(\x) \right] } \\
    &\le \norm{ \hat{f}_* }
      \sqrt{
        \sum_{S \in \cS_* \setminus \{ S\ps{l}_H \}}  
        \hat{\E}\ps{l}\left[ \chi_{S\ps{l}_H}(\x) \chi_S(\x) \right]^2
      }
      + B_w \norm{\balpha\ps{l}}
      \sqrt{ 
        \sum_{G \in \cH_l \setminus \cC_H} 
        \hat{\E}\ps{l}\left[ \chi_{S\ps{l}_H}(\x) \chi_{S\ps{l}_G}(\x) \right]^2
      } \\
    &\le \sqrt{ |\cS_*| } \eps
      + B_w \norm{\balpha\ps{l}} \sqrt{ M } \eps 
    \le \left( 1 + B_w \norm{\balpha\ps{l}} \right) \sqrt{M} \eps.
  \end{align*}
  By Assumption~\ref{assumption: training induction hypothesis}\ref{itm: training IH: features} and the condition $\norm{ f_* }_{L^2} \le 1$, we have $\norm{\balpha\ps{l}} \le 2^D \sqrt{M}$. 
  Hence, we can further rewrite the above as $\abs{ \cE\ps{l}_H } \le 2^{D+1} B_w M \eps$.
  Replacing $\delta_{\P}$ with $\delta_{\P} / (2M^2)$ and $\eps$ with $\eps_{\cE} / \big( 2^{D+1} B_w M \big)$ completes the proof.
\end{proof}

A standard convex optimization argument gives the following one-dimensional quadratic optimization lemma. 
\begin{lemma} 
  \label{lemma: ub: 1D quad opt}
  Fix arbitrary $a > 0$ and $b, c \in \R$ and define $\Loss(w) := a w^2 + b w + c$.
  Consider the update rule $w_{t+1} = w_t - \eta \Loss'(w_t)$ with $w_0 = 0$ and step size $\eta > 0$.
  If $\eta \le 1/(2 a)$, then $|w_t - w_*| \le \exp(- 2 a \eta t) |w_*|$, where $w_* = -b /(2a)$ is the unique minimizer of $\Loss$. 
\end{lemma}

Now, we combine the above three lemmas to show that one-pass coordinate descent will learn and deduplicate the correct features. 

\begin{lemma}[One-pass coordinate descent]
  \label{lemma: one-pass coordinate descent}
  Suppose that Assumption~\ref{assumption: training induction hypothesis} holds at layer $l$.
  Consider the $l$th iteration of the outer for loop of Algorithm~\ref{alg: training alg} (line~\ref{line: layerwise for loop}).
  Let $\eps_* \in (0, \kappa/2]$ and $\delta_{\P} \in (0, 1)$ denote the target accuracy and failure probability, respectively.
  Choose the algorithm parameters such that 
  \[
    \lambda_w = 2^{-D-1} \kappa, \quad 
    N 
    \gtrsim \frac{M^2 \log( M / \delta_{\P} )}{\kappa^{4D+4}}, \quad 
    \eta
    \le 2^{-2D},  \quad 
    T
    \gtrsim \frac{D \log(1/\kappa)}{ \kappa^{2D} \eta}.
  \]
  Then, with probability at least $1 - \delta_{\P}$, for every $H \in \cH_l$, the following holds after the $H$-th iteration of the inner for loop (line~\ref{line: coordinate descent for loop}): 
  \begin{enumerate}[(a)]
    \item If $\hat{f}_*(S\ps{l}_H) = 0$ or there exists some $H' \in \cH_l$ with $H' < H$ and $S\ps{l}_{H} = S\ps{l}_{H'}$, then $w\ps{l}_H = 0$.
    \item Otherwise, $w\ps{l}_H = (1 \pm 0.2) \hat{f}_*( S\ps{l}_H ) / \alpha\ps{l}_H$.
  \end{enumerate}
\end{lemma}
\begin{proof}
  Let $\eps_1, \eps_2 > 0$ be parameters to be determined later. 
  In addition, choose $N\ps{l}$ according to Lemma~\ref{lemma: ub: error term concatenation}, so that 
  $| \cE\ps{l}_H | \le \eps_{\cE} := \eps_2 \kappa$ uniformly for all $H \in \cH_l$. 

  First consider case~(b); that is, $H$ is the first among $\cC_H$ to appear in the inner for loop (line~\ref{line: coordinate descent for loop}) of Algorithm~\ref{alg: training alg}. 
  Since we initialize $\w\ps{l} = 0$, in this case, we can further simplify Lemma~\ref{lemma: ub: loss decomposition} as 
  \begin{align*}
    \Loss\ps{l}\big(\w\ps{l}\big)
    &= \frac{1}{2} \left(
        \hat{f}_*\big( S\ps{l}_H \big) 
        - \alpha\ps{l}_H w\ps{l}_H
      \right)^2 
      - \alpha\ps{l}_{H} w\ps{l}_{H} \cE\ps{l}_H
      + \left( \text{$w\ps{l}_H$-independent terms} \right) \\
    &= \frac{1}{2} \left( \alpha\ps{l}_H \right)^2 \left( w\ps{l}_H \right)^2
      - \left( \hat{f}_*\big( S\ps{l}_H \big) + \cE\ps{l}_H \right) \alpha\ps{l}_{H} w\ps{l}_H
      + \left( \text{$w\ps{l}_H$-independent terms} \right). 
  \end{align*}
  Recall that $\alpha\ps{l}_H := \prod_{i \in H} \alpha\ps{l}_i$ and $H$ is a subset of $[d + m]$ of size at most $D$.
  By Assumption~\ref{assumption: training induction hypothesis}\ref{itm: training IH: features} and the condition $\norm{f_*}_{L^2} \le 1$, we have $|\alpha\ps{l}_H| \le 1.2^D$. 
  Therefore, by Lemma~\ref{lemma: ub: 1D quad opt}, as long as we choose $\eta \le 2^{-2D}$, we will have, for any $t \ge 0$,  
  \begin{align*}
    \abs{ w\ps{l}_H(t) - \frac{ \hat{f}_*\big( S\ps{l}_H \big) + \cE\ps{l}_H }{ \alpha\ps{l}_H }   }
    &\le \exp\left(
        - \big( \alpha\ps{l}_H \big)^2 \eta t 
      \right)
      \abs{ \frac{ \hat{f}_*\big( S\ps{l}_H \big) + \cE\ps{l}_H }{ \alpha\ps{l}_H } }  \\
    &\le \exp\left(
        -  (\kappa/2)^{2D} \eta t 
      \right)
      \abs{ \frac{ \hat{f}_*\big( S\ps{l}_H \big) + \cE\ps{l}_H }{ \alpha\ps{l}_H } }, 
  \end{align*}
  where the second inequality comes from Assumption~\ref{assumption: training induction hypothesis}\ref{itm: training IH: features} and the condition that $f_*$ has signal strength $\kappa$. 
  As a result, for any $\eps_1 > 0$, we have 
  \[
    T 
    \ge \frac{\log(1/\eps_1)}{ (\kappa/2)^{2D} \eta}
    \quad\Rightarrow\quad 
    w\ps{l}_H(T ) 
    = (1 \pm \eps_1 ) \frac{ \hat{f}_*\big( S\ps{l}_H \big) + \cE\ps{l}_H }{ \alpha\ps{l}_H } 
    = (1 \pm 2 \eps_1 \pm 2 \eps_2 ) \frac{ \hat{f}_*\big( S\ps{l}_H \big) }{ \alpha\ps{l}_H }.
  \]
  In particular, when $\eps_1, \eps_2 \le 0.05$, this implies 
  \[
    \abs{ w\ps{l}_H(T)}
    = \abs{ (1 \pm 0.2 ) \frac{ \hat{f}_*\big( S\ps{l}_H \big) }{ \alpha\ps{l}_H } }
    \ge 0.8 \frac{ \kappa }{ 1.2^D } 
    \ge 2^{-D} \kappa.
  \]
  Hence, in line~\ref{line: truncation} of Algorithm~\ref{alg: training alg}, as long as we choose the truncation threshold $\lambda_w$ to be at most $2^{-D-1}\kappa$, we will not accidentally truncate $w\ps{l}_H(T)$ to $0$. 

  Now, consider case~(a). 
  First, suppose that there exists $H' < H$ with $S\ps{l}_{H'} = S\ps{l}_H$. 
  Let $H'$ be the earliest such coordinate.
  Then, by case~(b), after the $H'$-th iteration, we have $ w\ps{l}_{H'}(t) = (1 \pm \eps_1 \pm \eps_2 ) \hat{f}_*(S\ps{l}_H) / \alpha\ps{l}_{H'} $. 
  Therefore, the loss decomposition in Lemma~\ref{lemma: ub: loss decomposition} becomes
  \begin{multline*}
    \Loss\ps{l}\big(\w\ps{l}\big)
    = \frac{1}{2} \left(
        (\pm \eps_1 \pm \eps_2) \hat{f}_*\big( S\ps{l}_H \big) 
        - \alpha\ps{l}_{H} w\ps{l}_{H} 
        - \sum_{H'' \in \cC_H \setminus \{ H, H' \}} \alpha\ps{l}_{H''} w\ps{l}_{H''} 
      \right)^2 \\
      - \alpha\ps{l}_{H} w\ps{l}_{H} \cE\ps{l}_H
      + \left( \text{$w\ps{l}_H$-independent terms} \right).  
  \end{multline*}
  Note that this bound is also valid when $\hat{f}_*(S\ps{l}_H) = 0$. 
  Since we initialize $\w\ps{l} = 0$, $w\ps{l}_{H''} = 0$ for all $H'' > H$. 
  Assume as an induction hypothesis that $w\ps{l}_{H''} = 0$ for all $H'' \in \cC_H$ between $H'$ and $H$.
  Then, the above becomes 
  \begin{align*}
    \Loss\ps{l}\big(\w\ps{l}\big)
    &= \frac{1}{2} \left(
        (\pm \eps_1 \pm \eps_2) \hat{f}_*\big( S\ps{l}_H \big) 
        - \alpha\ps{l}_{H} w\ps{l}_{H} 
      \right)^2 
      - \alpha\ps{l}_{H} w\ps{l}_{H} \cE\ps{l}_H
      + \left( \text{$w\ps{l}_H$-independent terms} \right) \\
    &= \frac{1}{2} \big( \alpha\ps{l}_H  \big)^2 \big( w\ps{l}_H \big)^2
      - \left(
        (\pm \eps_1 \pm \eps_2) \hat{f}_*\big( S\ps{l}_H \big) 
        + \cE\ps{l}_H
      \right) 
      \alpha\ps{l}_H w\ps{l}_H
      + \left( \text{$w\ps{l}_H$-independent terms} \right). 
  \end{align*}
  By the same analysis, we know 
  \[
    \abs{ w\ps{l}_H(T) }
    = \abs{ (1 \pm \eps_1 ) \frac{ (\pm \eps_1 \pm \eps_2) \hat{f}_*\big( S\ps{l}_H \big) + \cE\ps{l}_H }{ \alpha\ps{l}_H } }
    \le \abs{ \frac{ 4 (\eps_1 + \eps_2)  }{ \alpha\ps{l}_H } }
    \le \frac{ 2^{D+2} (\eps_1 + \eps_2)  }{ \kappa^D }. 
  \]
  This implies that in order for line~\ref{line: truncation} to truncate $w\ps{l}_H(T)$ to $0$ when $\lambda_w = 2^{-D-1} \kappa$, it suffices to require 
  \[
    \frac{ 2^{D+2} (\eps_1 + \eps_2)  }{ \kappa^D }
    \le 2^{-D-1} \kappa
    \quad\Leftarrow\quad 
    \eps_1 + \eps_2 
    \le 2^{-2D-3} \kappa^{D+1}.
  \]
  Note that this constraint is stronger than the previous constraint imposed by case~(b), and the same argument remains valid for $\hat{f}_*(S\ps{l}_H) = 0$. 
  Hence, it suffices to keep this constraint and choose $\eps_1 = \eps_2 = 2^{-2D-4} \kappa^{D+1}$.
  In addition, since $w\ps{l}_H$ always moves toward the corresponding minimizer (cf.~Lemma~\ref{lemma: ub: 1D quad opt}) and we initialize it with $0$, we also have 
  \[
    \norm{\w\ps{l}}_\infty 
    \le \max_H \abs{ \frac{ \hat{f}_*\big( S\ps{l}_H \big) + \cE\ps{l}_H }{ \alpha\ps{l}_H } }
    \le 2 \cdot \frac{2^D}{\kappa^D}. 
  \] 
  Finally, using Lemma~\ref{lemma: ub: error term concatenation} and Lemma~\ref{lemma: ub: 1D quad opt}, we can translate the conditions on $\eps_1, \eps_2$ into the following constraint on the number of samples and steps: 
  \[
    N 
    \gtrsim \frac{M^2 \log( M / \delta_{\P} )}{\kappa^{4D+4}}, \quad 
    T 
    \gtrsim \frac{D \log(1/\kappa)}{ \kappa^{2D} \eta}.
  \]
\end{proof}

With this lemma, we can now show that the induction hypothesis Assumption~\ref{assumption: training induction hypothesis} holds and estimate the corresponding $\alpha$'s.

\begin{lemma}[Induction step]
  \label{lemma: ub: induction step}
  Suppose that Assumption~\ref{assumption: training induction hypothesis} holds at layer $l \in [L]$ and choose the algorithm parameters according to Lemma~\ref{lemma: one-pass coordinate descent}.
  Then, Assumption~\ref{assumption: training induction hypothesis} also holds at layer $l+1$.
\end{lemma}
\begin{proof}
  Recall from the discussion after Fact~\ref{fact: ub: base case} that $\cF_{l+1/2} := \{ S\ps{l}_H \}_{H \in  \cH_l}$ contains all elements in $\cK_{l+1}$.
  By Lemma~\ref{lemma: one-pass coordinate descent}, for any $S \in \cK_{l+1} \cap \cS_*$, there exists a unique coordinate $H_S \in [M]$ with $w\ps{l}_{H_S} = (1 \pm 0.2) \hat{f}_*(S) / \alpha\ps{l}_{H_S}$. 
  Combining this identity with \eqref{eq: zlH} gives
  \[
    w\ps{l}_{H_S} z\ps{l}_{H_S}
    = (1 \pm 0.2) \hat{f}_*(S) \chi_S(\x).
  \]
  In other words, $\alpha\ps{l+1}_{H_S} = (1 \pm 0.2) \hat{f}_*(S)$.
  Note that the same bound and uniqueness result still hold for $S \in \cS_* \setminus \cK_{l+1} $, but we do not guarantee the existence of $H_S$. 
  Also, by Lemma~\ref{lemma: one-pass coordinate descent}, all other coordinates of $\w\ps{l}$ and, therefore, the corresponding entries in $\w\ps{l} \odot \z\ps{l}$, are always $0$. 
  In particular, this implies that the number of nonzero entries in $\w\ps{l} \odot \z\ps{l}$ is at most $|\cS_*| \le m$. 
  Hence, $\Select$ will never discard needed entries. 
  Recall that $\x\ps{l+1}$ is obtained by concatenating $\x$ with $\y\ps{l}$. 
  Therefore, each of the sets $\{1\}, \dots, \{d\}$ can correspond to at most two entries in $\x\ps{l+1}$, with the corresponding $\alpha$ equal either to $1$ or to the corresponding Fourier coefficient.
\end{proof}

Our main optimization theorem is a corollary of this lemma. 
\begin{proof}[Proof of Theorem~\ref{thm: upper bound}]
  First, by Fact~\ref{fact: ub: base case} and Lemma~\ref{lemma: ub: induction step}, Assumption~\ref{assumption: training induction hypothesis} holds for all $l \in [L]$. 
  For the last layer, to achieve the target accuracy $\eps_*$, we need more samples and steps. 
  Let $\tilde{\eps}_* > 0$ be a parameter to be determined later. 
  The proof of Lemma~\ref{lemma: one-pass coordinate descent}, \textit{mutatis mutandis}, shows that to ensure $w\ps{L}_H = (1 \pm \tilde{\eps}_*) \hat{f}_*(S\ps{L}_H) / \alpha\ps{L}_H$ for every $H$ belonging to case (b), it suffices to choose $\eps_1 \vee \eps_2 \le \tilde{\eps}_* \wedge 2^{-2D-4} \kappa^{D+1}$, where $\eps_1, \eps_2$ are defined in the proof of Lemma~\ref{lemma: one-pass coordinate descent}. 
  In particular, this translates into the constraints:
  \[
    N 
    \gtrsim \frac{ M^2 \log( M / \delta_{\P} ) }{
      \kappa^{2D+2} 
      \left( \tilde{\eps}_*^2 \wedge \kappa^{2D+2} \right) 
    } , \quad 
    T 
    \gtrsim \frac{
      \log\left(\tilde{\eps}_*\inv \vee 1/\kappa \right)
    }{ 
      \kappa^{2D} \eta
    }.
  \] 
  These two conditions are stronger than those required by earlier layers. 
  For simplicity, we will just replace all earlier requirements with these two. 
  By Parseval's theorem, conditioned on the above event, we have 
  \[
    \norm{ f_* - f }_{L^2} 
    = \norm{ \hat{f}_* - \hat{f} }_2 
    \le \sqrt{|\cS_*|} \tilde{\eps}_*. 
  \] 
  Finally, taking a union bound over all $L$ layers and replacing $\delta_{\P}$ with $\delta_{\P}/L$ and $\tilde{\eps}_*$ with $\eps_*/\sqrt{|\cS_*|}$ completes the proof.
\end{proof}

\section{Compiling the learner network}
\label{appendix: approx ub net with lb net}

In this section, we show that we can efficiently compile the learner network from our upper bound (cf.~Theorem~\ref{thm: upper bound} and \eqref{eq: learner network}) into a network whose architecture matches our lower bound (cf.~\eqref{eq: lb: network definition} and Assumption~\ref{assumption: lower bound assumption}). 

First, recall from Appendix~\ref{appendix: upper bound proofs} that Assumption~\ref{assumption: training induction hypothesis} holds for every $l \in [L]$.
In particular, for any nontrivial coordinate $k \in I_l$, we have 
\[
  x\ps{l}_k = \alpha\ps{l}_k \chi_{S\ps{l}_k}(\x), \quad 
  \alpha\ps{l}_k \ne 0, \quad S\ps{l}_k \subset [d].
\]
Instead of computing $\x\ps{l}$ using a layer satisfying Assumption~\ref{assumption: lower bound assumption}, we compute a normalized version of it.
Specifically, we define 
\[
  \tilde{x}\ps{l}_k(\x)
  = \begin{cases}
    \chi_{S\ps{l}_k}(\x), & k \in I_l, \\
    0, & \text{otherwise},
  \end{cases}
  \quad \forall \x \in \{\pm 1\}^d,\ k \in [d+m]. 
\]
By our analysis in Appendix~\ref{appendix: upper bound proofs} and the fact that the first $d$ coordinates of $\tilde{\x}\ps{l}(\x)$ are always $\x$, for any $j \in [d + m]$, either $\tilde{x}\ps{l+1}_j \equiv 0$, or there exists $H\ps{l}_j \subset I_l$ such that 
\[
  \tilde{x}\ps{l+1}_j 
  \equiv \chi_{S\ps{l}_{H\ps{l}_j}} 
  \quad\text{where}\quad 
  S\ps{l}_{H\ps{l}_j} 
  := \bigoplus_{i \in H\ps{l}_j} S\ps{l}_i. 
\]
For each $l \in [L]$, let $\A\ps{l} \in \{0, 1\}^{(d+m) \times (d+m)}$ be the corresponding incidence matrix: 
\begin{equation}
  \label{eq: compile: A}
  A\ps{l}_{ji} := \indi\braces{ i \in H\ps{l}_j  },
\end{equation}
with the convention that $A\ps{l}_{ji} = 0$ if $j \notin I_{l+1}$.
Define the maximum fan-out of the network as 
\begin{equation}
  \label{eq: fan-out Delta}
  \Delta := \max_{l \in [L]} \max_{i \in [d+m]} \abs{ \braces{ j \in [d+m] \,:\, i \in H\ps{l}_j } }.
\end{equation}
In words, $\Delta$ is the maximum number of times a given input coordinate is used to construct the outputs of a layer.
The following are the main results of this section. 

\begin{lemma}[Compiling the network]
  \label{lemma: compile: main}
  Suppose that the target function $f_*: \{\pm 1\}^d \to \R$ belongs to $\HS_{L, D, m, \kappa}$ with $L \le d$, $D = O(1), m \le \poly(d), \kappa \ge 1/\poly(d)$ and $\norm{f_*}_{L^2} \le 1$.
  Let our learner network $(\x\ps{l})_l, f$ be defined by \eqref{eq: learner network} and suppose that it is (successfully) trained to fit $f_*$ to $\eps_*$-accuracy with $\eps_* = o(1)$ using Algorithm~\ref{alg: training alg}.\footnote{
    We say the model is successfully trained if all high-probability events used in the proof occur. 
    By (the proof of) Theorem~\ref{thm: upper bound}, this happens with probability at least $1 - \delta_{\P}$.
  } 
  Let $\Delta$ denote the maximum fan-out of the learner network (cf.~\eqref{eq: fan-out Delta}). 

  Then, there exist activations $\bphi_0, \bphi_1, \dots, \bphi_{L+1}: \R^{d+m} \to \R^{d+m}$ and weight matrices $\W\ps{0} \in \R^{(d+m) \times d}$, $\W\ps{1}, \dots, \W\ps{L+1} \in \R^{(d+m) \times (d+m)}$, $\a \in \R^{d+m}$ such that the following hold: 
  \begin{enumerate}[(a)]
    \item They satisfy Assumption~\ref{assumption: lower bound assumption}\ref{itm: lb: op norm}-\ref{itm: lb: analytic, strip growth} with $C_W = 2 \vee (\pi/2) \sqrt{D \Delta}$, $c_\phi = 1$, $C_\phi = 2$, $m_\phi = 1$. 
    \item Define $\hat{\x}\ps{1}, \dots, \hat{\x}\ps{L+2}: \{\pm 1\}^d \to \R^{d + m}$ inductively by 
      \[
        \hat{\x}\ps{1}(\x) = \bphi_0\left( \W\ps{0} \x \right), \quad 
        \hat{\x}\ps{l+1}(\x) = \bphi_l\left( \W\ps{l}\hat{\x}\ps{l}(\x) \right), \quad 
        \forall l \in [L+1].
      \]
      We have $f(\x) = \inprod{\a}{\hat{\x}\ps{L+2}(\x)}$ for all $\x \in \{\pm 1\}^d$.
  \end{enumerate}
  Moreover, the activation functions and weights can be constructed from the trained network in polynomial time. 
\end{lemma}

\begin{corollary}
  \label{cor: compile: deep quad}
  Let our learner network $(\x\ps{l})_l, f$ be defined by \eqref{eq: learner network} and suppose that it is (successfully) trained to fit a deep quadratic function to $\eps_*$-accuracy with $\eps_* = o(1)$ using Algorithm~\ref{alg: training alg}. 
  Then, $\Delta \le 5$. 
  As a result, Lemma~\ref{lemma: compile: main} holds, and the constructed network satisfies all conditions in Assumption~\ref{assumption: lower bound assumption} except for the condition $L = O(1)$.
\end{corollary}

To prove Lemma~\ref{lemma: compile: main}, first we show that each output feature $\tilde{x}\ps{l+1}_j$ can be represented using a single neuron.
\begin{lemma}
  \label{lemma: compile: single neuron}
  Fix $l \in [L]$ and $H \subset I_l$.
  Define $\w_H \in \R^{d+m}$ by $w_{H, i} := (\pi / 2) \indi\{ i \in H \}$, and let $\phi_H(z) := \cos( \pi |H|/2 - z )$.
  Then, we have 
  \[
    \phi_H\left( \w_H \cdot \tilde{\x}\ps{l}(\x) \right)
    = \chi_{S\ps{l}_H}(\x), \quad 
    \forall \x \in \{\pm 1\}^d.
  \]
\end{lemma}
\begin{proof}
  For notational simplicity, we drop the superscript $l$ in this proof. 
  Let $K_H(\x) := \abs{ \{ i \in H \,:\, \tilde{x}_i(\x) = -1 \} }$.
  By construction, we have 
  \[
    \inprod{\w_H}{\tilde{\x}(\x)}
    = \sum_{i \in I_l} w_{H, i} \tilde{x}_i(\x)
    = \frac{\pi}{2} \sum_{i \in H} \chi_{S_i}(\x)
    = \frac{\pi}{2} \left( |H| - 2 K_H(\x) \right). 
  \]
  Therefore, 
  \[
    \phi_H( \w_H \cdot \tilde{\x}(\x) )
    = \cos\left( \frac{\pi}{2} |H| - \frac{\pi}{2} \left( |H| - 2 K_H(\x) \right) \right)
    = \cos( \pi K_H(\x) )
    = \begin{cases}
      1,  & \text{$K_H(\x)$ is even}, \\
      -1, & \text{$K_H(\x)$ is odd}.
    \end{cases}
  \]
  The quantity $K_H(\x)$ is even if and only if $\prod_{i \in H} \tilde{x}_i(\x) = 1$, which holds if and only if $\chi_{S_H}(\x) = 1$.
  In other words, we have $\phi_H( \w_H \cdot \tilde{\x}(\x) ) = \chi_{S_H}(\x)$.
\end{proof}

\begin{lemma}
  \label{lemma: compile: one layer}
  For any $l \in [L]$, there exist $\bphi_l$ and $\W\ps{l}$ satisfying Assumption~\ref{assumption: lower bound assumption} with $C_W = (\pi/2) \sqrt{D \Delta}$, $c_\phi = 1$, $C_\phi = 2$, $m_\phi = 0$ such that 
  \[
    \tilde{\x}\ps{l+1}(\x) 
    := \bphi_l\left( \W\ps{l} \tilde{\x}\ps{l}(\x)  \right), \quad \forall \x \in \{\pm 1\}^d. 
  \]
\end{lemma}
\begin{proof}
  For notational simplicity, we write $S\ps{l}_j := S\ps{l}_{H\ps{l}_j}$ in this proof. 
  Recall that for any $j \in [d+m]$, we have 
  \[
    \tilde{x}\ps{l+1}_j(\x) 
    = \begin{cases}
      \chi_{S\ps{l}_j}(\x), & j \in I_{l+1}, \\
      0, & \text{otherwise}, 
    \end{cases} 
    \quad \forall \x \in \{\pm 1\}^d. 
  \] 
  For $j \notin I_{l+1}$, we simply set $\phi_{l, j} \equiv 0$ and $\w\ps{l}_j = 0$. 
  For $j \in I_{l+1}$, we choose 
  \[
    \w\ps{l}_j = \left( \frac{\pi}{2} \indi\braces{ i \in H\ps{l}_j } \right)_{i \in [d+m]}, \quad 
    \phi_{l, j}(z) = \cos\left( \frac{\pi}{2}\abs{ H\ps{l}_j } - z \right).
  \]
  Then, by Lemma~\ref{lemma: compile: single neuron}, we have 
  \[
    \tilde{\x}\ps{l+1}(\x) 
    := \bphi_l\left( \W\ps{l} \tilde{\x}\ps{l}(\x)  \right), \quad \forall \x \in \{\pm 1\}^d. 
  \]
  Now, we upper bound the operator norm of $\W\ps{l}$. 
  Recall the definition of $\A\ps{l} \in \{0, 1\}^{(d+m) \times (d+m)}$ from the beginning of this section.
  Note that we have 
  \[
    \norm{\W\ps{l}}_{\OP}
    = \frac{\pi}{2} \norm{\A\ps{l}}_{\OP}
    \le \frac{\pi}{2} \sqrt{ \norm{\A\ps{l}}_1 \norm{\A\ps{l}}_\infty  },
  \]
  where $\norm{\A}_1 := \max_{i} \sum_{j} |A_{ji}|$ and $\norm{\A}_\infty := \max_j \sum_i |A_{ji}|$.
  Since $|H\ps{l}_j| \le D$, we have $\norm{\A\ps{l}}_\infty \le D$. 
  In addition, by the definition of maximum fan-out (cf.~\eqref{eq: fan-out Delta}), we have $\norm{\A\ps{l}}_1 \le \Delta$. 
  Therefore, 
  \[
    \norm{\W\ps{l}}_{\OP}
    \le \frac{\pi}{2} \sqrt{D \Delta}.
  \]
  Finally, we show that $\bphi_l$ satisfies Assumption~\ref{assumption: lower bound assumption}\ref{itm: lb: sigma derivatives bounds}-\ref{itm: lb: analytic, strip growth} with $c_\phi = 1$, $C_\phi = 2$ and $m_\phi = 0$.
  Since our activation functions are shifted cosines, it is clear that Assumption~\ref{assumption: lower bound assumption}\ref{itm: lb: sigma derivatives bounds} holds as long as $C_\phi \ge 1$. 
  In addition, for any $x_0 \in \R$, $z \mapsto \cos(x_0 + z)$ is entire and satisfies $|\cos( a + x_0 + i b )|\le \cosh|b|$ for all $a, b \in \R$. 
  Take $c_\phi = 1$. 
  Then, for any $|b| \le 1/2$, we have $\abs{ \cos(a + x_0 + i b) } \le \cosh|b| \le 2$.
  In other words, Assumption~\ref{assumption: lower bound assumption}\ref{itm: lb: analytic, strip growth} holds with $c_\phi = 1$, $C_\phi = 2$ and $m_\phi = 0$.
\end{proof}

Now, we are ready to prove the main results of this section.
\begin{proof}[Proof of Lemma~\ref{lemma: compile: main}]
  First, by Lemma~\ref{lemma: compile: one layer}, we can choose $\bphi_1, \dots, \bphi_L$ and $\W\ps{1}, \dots, \W\ps{L}$ satisfying the requirements such that 
  \[
    \tilde{\x}\ps{l+1}(\x) 
    := \bphi_l\left( \W\ps{l} \tilde{\x}\ps{l}(\x)  \right), 
    \quad \forall \x \in \{\pm 1\}^d, \, l \in [L]. 
  \]
  Hence, to complete the proof, it suffices to construct the $0$th layer and the output layer(s). 
  Recall that the raw input is $\x \in \{\pm 1\}^d$ and $\tilde{\x}\ps{1}(\x) = \x \circ \mbf{0} \in \{0, \pm 1\}^{d + m}$.
  Hence, to ensure $\hat{\x}\ps{1} \equiv \tilde{\x}\ps{1}$, it suffices to choose $\bphi_0 = \mrm{Id}$ and $\W\ps{0} = [ \Id_d \; \mbf{0} ]\trans \in \R^{(d+m) \times d}$.
  Then, define $(\hat{\x}\ps{l})_{l \in [L+1]}$ as in Lemma~\ref{lemma: compile: main}.
  We have $\hat{\x}\ps{l} \equiv \tilde{\x}\ps{l}$ for all $l \in [L+1]$.
  In particular, this implies that for every $S \in \cS_*$, there exists some $j_S \in [d + m]$ such that $\hat{x}\ps{L+1}_{j_S} \equiv \chi_S$. 
  In addition, by the proof of Theorem~\ref{thm: upper bound}, we know that after training, the trained model satisfies
  \[
    f 
    = \sum_{S \in \cS_*} \tilde{a}_S \chi_S,
    \quad\text{where }\, \abs{ \tilde{a}_S  - \hat{f}_*(S) } \le \eps_* / \sqrt{|\cS_*|},
    \; \forall S \in \cS_*.
  \]
  Thus, in principle, it suffices to set the entry of the output weight $\a$ associated with each $S$ equal to the corresponding $\tilde{a}_S$.
  Note that because of these small errors, the norm of this $\a$ can potentially be slightly larger than $1$. 
  To fix this, we set $\W\ps{L+1} = 2 \Id_{d+m}$ and $\bphi_{L+1} = \mrm{Id}$, and then set the corresponding entry in $\a$ to be $\tilde{a}_S / 2$.
  Then, we have $\norm{\a}_2 = \norm{\tilde{\a}}/2 \le (1 + \eps_*)/2 \le 1$.
  Also, note that the identity activations satisfy Assumption~\ref{assumption: lower bound assumption} with $c_\phi = 1, C_\phi = 2, m_\phi = 1$.

  Finally, to verify that these weights and activations can be constructed efficiently from the original trained network, observe that the incidence matrix \eqref{eq: compile: A} and output weights can be constructed by recording the original index of the intermediate features and inspecting the nonzero entries of the learner weights.  
\end{proof}

\begin{proof}[Proof of Corollary~\ref{cor: compile: deep quad}]
  It suffices to show that the maximum fan-out $\Delta$ is at most $5$.
  Without loss of generality, we may assume that the target function is the base deep quadratic function.

  Recall that $\x\ps{l} = \x \circ \y\ps{l-1}$. 
  In the first layer, the degree-$1$ feature will be copied once and used to construct the (disjoint) degree-$2$ features. 
  Hence, the maximum fan-out of the first layer is~$2$.

  Now, fix $l \ge 2$.
  Recall from our induction hypothesis (Assumption~\ref{assumption: training induction hypothesis}) that the features stored in $\y\ps{l}$ are unique. 
  In addition, note that each input feature can be used at most $5$ times. 
  To see this, consider some $P \subset [d]$ in an intermediate representation. 
  Let $Q$ denote its sibling, $A, B$ its two unique children, and $R = P \cup Q$ its (unique) parent. 
  Then, $P$ can only be potentially used to construct the output features $P$, $A = P \oplus B$, $B = P \oplus A$, $Q = R \oplus P$ and $R = P \oplus Q$. 
  Hence, $\Delta \le 5$.

  To complete the proof, recall from Fact~\ref{fact: deep quad function} that the local degree and width of any deep quadratic function are $2$ and $d-1$, respectively.   
\end{proof}

\section{Further discussion and examples of hierarchical staircase functions}
\label{appendix: examples}

In this section, we give further examples of hierarchical staircase functions and discuss their relationship to and differences from related function classes.

\subsection{Staircase-type functions}
\label{appendix: examples: staircase}

\paragraph{Staircase function.}
\cite{abbe_staircase_2021} introduce (basic) staircase functions, which are defined as Boolean functions $f = \sum_{S \subset [d]} \hat{f}(S) \chi_S$ satisfying the property that for every $S \in \cS_f$ with $|S| \ge 2$, there exists $S' \subset S$ such that $S' \in \cS_f$ and $|S \setminus S'| = 1$. 
In words, every set $S$ in the Fourier spectrum can be constructed by taking the union of another smaller set $S'$ in the spectrum and one degree-$1$ monomial. 
It is clear that such functions are also hierarchical staircase functions with local degree $2$ and depth at most the size of the largest set in $\cS_f$ plus one. 

\paragraph{Merged-staircase function.}
Later, \cite{abbe_merged-staircase_2022} generalize staircase functions to merged-staircase functions by replacing $S'$ with a union of earlier sets.
Formally, a Boolean function $f$ is said to be a merged-staircase function if there is an ordering $\{ S_1, \dots, S_r \}$ of $\cS_f$ such that 
\begin{equation}
  \label{eq: merged staircase}
  \textstyle
  \abs{ S_i \setminus \bigcup_{j=1}^{i-1} S_j } \le 1, \quad \forall i \in [r],
\end{equation}
and the optimization guarantees in \cite{abbe_merged-staircase_2022} further require $f$ to be $O(1)$-sparse, i.e., $\big| \bigcup_{S \in \cS_f} S \big| = O(1)$.
The function classes and results in this paper and \cite{abbe_merged-staircase_2022} are not directly comparable. 
An obvious difference is that we allow the function to have genuinely high-degree terms, while the sparsity condition in \cite{abbe_merged-staircase_2022} effectively allows only constant-degree terms. 
Hence, in principle, all $k$-sparse merged-staircase functions are depth-$2$ hierarchical staircase functions with local degree $k$.
Even if we drop the sparsity constraint, these two function classes are still not comparable, since $\bigcup_{j=1}^{i-1} S_j$ can potentially involve a superconstant number of earlier sets.
For example, take $S_i = \{i\}$ for $i \in [d/2]$ and $S_{d/2+1} = \{1, \dots, d/2\}$. The collection $\{S_1, \dots, S_{d/2+1}\}$ satisfies \eqref{eq: merged staircase}, but it is not a hierarchical staircase function with constant local degree.
However, if we restrict the number of $\{ S_j \}_{j < i}$ that can be used by $S_i$ to be a constant $D$ and replace the union and set difference with symmetric difference in \eqref{eq: merged staircase}, then the resulting functions will be hierarchical staircase functions with local degree $D+1$. 
In addition, we believe that some constraint on sparsity and/or on the number of usable sets is necessary for efficient learning.
To see this, note that $S_i = \{i\}$ for $i \in [d]$ and $S_{d+1} = \{ 1, \dots, d/2 \}$ also satisfy \eqref{eq: merged staircase}, and learning $S_{d+1}$ after learning $(S_i)_{i \in [d]}$ is essentially learning $(d/2)$-parity, which is known to be hard for statistical query models (\cite{kearns_efficient_1998}) and certain gradient-based methods (\cite{shalev-shwartz_failures_2017}).

\paragraph{Leap complexity.}
\cite{abbe_sgd_2023} further generalize the concept of merged-staircase functions by allowing the right-hand side of \eqref{eq: merged staircase} to be a constant larger than $1$. 
The smallest possible such constant is called the leap complexity. 
As in the preceding discussion, if we restrict the number of usable sets on the left-hand side of \eqref{eq: merged staircase} to be constant and replace the union and set difference with symmetric difference, then functions with constant leap complexity are hierarchical staircase functions whose local degree is at most the number of usable sets plus the leap complexity.

\subsection{Globally read-once decision trees}

A \tnbf{Boolean decision tree} $f_T$ over $\{\pm 1\}^d$ is a rooted binary tree whose leaves are labeled by elements of $\{\pm 1\}$ and whose internal nodes $v$ are labeled by coordinates $i_v \in [d]$; the two outgoing edges of each internal node are labeled by the elements of $\{\pm 1\}$.
To evaluate $f_T(\x)$ for $\x \in \{\pm 1\}^d$, we start at the root. At an internal node $v$, we query $x_{i_v}$ and follow the edge labeled by its value.
The output is the value of the leaf we eventually reach. 
We say a decision tree is \tnbf{globally read-once} if each coordinate of $\x$ can appear at most once in the decision tree. 
For more background on decision trees, one may refer to Section~3.2 of \cite{odonnell_analysis_2014}.

\begin{lemma}
  Let $f_T$ be a depth-$h$ globally read-once Boolean decision tree, where $h \ge 1$.
  Then, $f_T \in \HS_{h, 2, m, \kappa}$ with $m \le 4^h$ and $\kappa \ge 2^{-h}$.
\end{lemma}
\begin{proof}
  We first verify the hierarchical staircase property by induction on $h$.
  Let $i$ be the coordinate queried at the root, and let $f_-$ and $f_+$ be the functions computed by the subtrees reached when $x_i=-1$ and $x_i=+1$, respectively.
  Then
  \begin{equation}
    \label{eq: globally read-once recursion}
    f_T(\x)
    = \frac{1-x_i}{2} f_-(\x) + \frac{1+x_i}{2} f_+(\x).
  \end{equation}
  Let $V_-$ and $V_+$ denote the sets of coordinates queried in the two subtrees.
  Since $T$ is globally read-once, $V_-$ and $V_+$ are disjoint and neither contains $i$.
  Note that for any Boolean function $f$ and $S \subset [d]$, $\hat{f}(S) = 0$ if there exists some $i \in S$ on which $f$ does not depend. 
  Hence, it follows from \eqref{eq: globally read-once recursion} that, for each $a \in \{\pm 1\}$ and each nonempty $S \in \cS_{f_a}$, 
  \[
    i \notin S, \qquad
    \hat f_T(S) = \frac{1}{2}\hat f_a(S), \qquad
    \hat f_T(S \cup \{i\}) = \frac{a}{2}\hat f_a(S).
  \]
  In particular, both $S$ and $S \cup \{i\}$ belong to $\cS_{f_T}$.
  Moreover, apart from $\varnothing$ and $\{i\}$, every set in $\cS_{f_T}$ is of one of these two forms.
  In particular, in the second case, $S \cup \{i\} = S \oplus \{i\}$ can be constructed by combining an existing feature and a degree-$1$ term. 

  If $h=1$, both subtrees are constant, so $\cS_{f_T} \subseteq \{\varnothing,\{i\}\} \subseteq \cK_1$.
  Now suppose $h\ge 2$.
  By the induction hypothesis, every nonempty set in the spectrum of either subtree can be constructed within $h-1$ rounds.
  The same construction remains valid for $f_T$, since all of its nonempty intermediate sets also belong to $\cS_{f_T}$ by the preceding display.
  Thus, every such $S$ belongs to $\cK_{h-1}$ for $f_T$.
  Since $\{i\}\in\cK_1\subseteq\cK_{h-1}$ and
  \(
    S \cup \{i\} = S \oplus \{i\},
  \)
  every set of the second form belongs to $\cK_h$.
  We conclude that $\cS_{f_T}\subseteq\cK_h$, so $f_T$ has depth at most $h$ and local degree at most $2$.

  It remains to bound the width and signal strength.
  For each leaf $l$, let $b_l\in\{\pm1\}$ be its label, let $P_l$ be the set of coordinates queried on the root-to-$l$ path, and let $a_{l,j}\in\{\pm1\}$ be the edge label followed when coordinate $j\in P_l$ is queried.
  Since the leaf events partition the hypercube,
  \[
    f_T(\x)
    = \sum_{l} b_l
      \prod_{j\in P_l}\frac{1+a_{l,j}x_j}{2}.
  \]
  A leaf at depth $t$ contributes at most $2^t\le 2^h$ Fourier monomials, and a depth-$h$ binary tree has at most $2^h$ leaves.
  Hence,
  \[
    |\cS_{f_T}| \le \sum_l 2^{|P_l|} \le 4^h.
  \]
  Finally, after expanding the preceding leaf representation, every Fourier coefficient is a sum of terms of the form $\pm 2^{-|P_l|}$.
  Since $|P_l|\le h$, every Fourier coefficient is an integer multiple of $2^{-h}$.
  Therefore, every nonzero Fourier coefficient has magnitude at least $2^{-h}$, which proves the claimed lower bound on $\kappa$.
\end{proof}

\subsection{Addressing functions}

Fix $k \in \bbN_+$.
Let $\a = (a_1, \dots, a_k) \in \{\pm 1\}^k$ be the address variables and let $\y = (y_b)_{b \in \{\pm 1\}^k} \in \{\pm 1\}^{2^k}$ be the data variables.
The $k$-bit \tnbf{addressing function} is
\begin{equation}
  \label{eq: addressing function}
  \Addr_k(\a, \y)
  := y_{\a}
  = \sum_{b \in \{\pm 1\}^k}
    y_b \prod_{j=1}^k \frac{1 + b_j a_j}{2}.
\end{equation}
In words, $\Addr_k$ returns the data bit indexed by the address $\a$.

\begin{lemma}
  \label{lemma: addressing function is hs}
  \(
    \Addr_k \in \HS_{k+1, 2, 4^k, 2^{-k}}.
  \)
\end{lemma}
\begin{proof}
  Expanding \eqref{eq: addressing function} gives
  \begin{equation}
    \label{eq: addressing fourier expansion}
    \Addr_k(\a, \y)
    =
    2^{-k}
    \sum_{b \in \{\pm 1\}^k}
    \sum_{A \subseteq [k]}
      \left(\prod_{j \in A} b_j\right)
      y_b \prod_{j \in A} a_j.
  \end{equation}
  The $4^k$ monomials in \eqref{eq: addressing fourier expansion} are distinct, and every coefficient has magnitude $2^{-k}$.
  Hence, $\Addr_k$ has width $4^k$ and signal strength $2^{-k}$.

  For $A \subseteq [k]$ and $b \in \{\pm 1\}^k$, let $S_{A,b}$ denote the support of the monomial
  $
    y_b \prod_{j \in A} a_j.
  $
  We show by induction on $|A|$ that $S_{A,b} \in \cK_{|A|+1}$.
  The base case follows because $S_{\varnothing,b} = \{y_b\} \in \cK_1$.
  If $A$ is nonempty, fix any $j \in A$.
  Then
  \[
    S_{A,b}
    = S_{A \setminus \{j\},b} \oplus \{a_j\}.
  \]
  The first set belongs to $\cK_{|A|}$ by the induction hypothesis, the singleton $\{a_j\}$ belongs to $\cK_1 \subseteq \cK_{|A|}$, and $S_{A,b}$ belongs to the spectrum by \eqref{eq: addressing fourier expansion}.
  Therefore, $S_{A,b} \in \cK_{|A|+1}$.
  Since $|A| \le k$, the full spectrum is contained in $\cK_{k+1}$, proving the hierarchical staircase claim.
\end{proof}

\end{document}